\documentclass[a4paper,12pt]{article} 

\title{Improved Margin Generalization Bounds for Voting Classifiers}

\author{
  Mikael Møller Høgsgaard \\
  Aarhus University \\
  \texttt{hogsgaard@cs.au.dk}
  \and
  Kasper Green Larsen \\
  Aarhus University \\
  \texttt{larsen@cs.au.dk}
}
\date{}
\usepackage{fullpage}
\usepackage{mathtools,bbm}
\usepackage{bm}
\usepackage{amsfonts}
\usepackage{amsmath}
\usepackage{amssymb}
\usepackage{amsthm}

\usepackage{thmtools}

\usepackage[colorlinks = true,
            linkcolor = blue,
            urlcolor  = blue,
            citecolor = blue,
            anchorcolor = blue]{hyperref}
\usepackage[capitalise,nameinlink]{cleveref}
\usepackage{MnSymbol}

\newtheorem{theorem}{Theorem}
\newtheorem{lemma}[theorem]{Lemma}
\newtheorem{corollary}[theorem]{Corollary}

\newcommand{\cc}{64}

\newcommand{\eps}{\varepsilon}
\newcommand{\ind}{\mathbbm{1}}

\newcommand{\Net}{N}

\DeclareMathOperator*{\e}{\mathbb{E}}
\DeclareMathOperator*{\p}{\mathbb{P}}
\renewcommand{\Pr}{\p}
\DeclareMathOperator{\Ln}{\mathrm{Ln}}

\DeclareMathOperator{\fat}{\mathrm{fat}}
\DeclareMathOperator{\sign}{\mathrm{sign}}
\newcommand{\cff}{\cF_{\left\lceil\gamma\right\rceil}}
\newcommand{\cffs}{f_{\left\lceil\gamma\right\rceil}}
\DeclareMathOperator{\ls}{\mathcal{L}}
\DeclareMathOperator{\dlh}{\Delta(\cH)}
\DeclareMathOperator{\maj}{\mathrm{Maj}}
\DeclareMathOperator{\margin}{margin}
\DeclareMathOperator{\func}{\Gamma}

\makeatletter
\newcommand{\alphacmd@factory}[1]{}
\newcounter{alphacmdcounter}
\newcommand{\GenerateAlphabetCmds}[2]{%
    \renewcommand{\alphacmd@factory}[1]{%
        \expandafter\providecommand\csname #1##1\endcsname{{#2{##1}}}%
    }
    \setcounter{alphacmdcounter}{0}
    \loop
        \stepcounter{alphacmdcounter}
        \edef\alphacmd@ID{\@Alph\c@alphacmdcounter}
        \expandafter\alphacmd@factory\alphacmd@ID
    \ifnum\thealphacmdcounter<26
    \repeat
}
\newcommand{\GenerateAlphabetCmdsLower}[2]{%
    \renewcommand{\alphacmd@factory}[1]{%
        \expandafter\providecommand\csname #1##1\endcsname{{#2{##1}}}%
    }
    \setcounter{alphacmdcounter}{0}
    \loop
        \stepcounter{alphacmdcounter}
        \edef\alphacmd@ID{\@alph\c@alphacmdcounter}
        \expandafter\alphacmd@factory\alphacmd@ID
    \ifnum\thealphacmdcounter<26
    \repeat
}
\makeatother

\GenerateAlphabetCmds{c}{\mathcal}
\GenerateAlphabetCmdsLower{c}{\mathcal}

\GenerateAlphabetCmds{r}{\mathbf}
\GenerateAlphabetCmdsLower{r}{\mathbf}

\begin{document}
\maketitle

\begin{abstract}
  In this paper we establish a new margin-based generalization bound
  for voting classifiers, refining existing results and yielding
  tighter generalization guarantees for widely used boosting
  algorithms such as AdaBoost~\cite{adaboostyoavfreund}. Furthermore,
  the new margin-based generalization bound enables the derivation of
  an optimal weak-to-strong learner: a Majority-of-3 large-margin
  classifiers with an expected error matching the theoretical lower
  bound. This result provides a more natural alternative to the
  Majority-of-5 algorithm 
  by~\cite{manyfacesofoptimalweaktostronglearning}, and matches the
  Majority-of-3 result by~\cite{majorityofthree} for the realizable prediction model.

\end{abstract}

\section{Introduction}
Creating ensembles of classifiers is a classic technique in machine learning, and thus studying generalization of such ensembles is natural. An example of the success of ensembles of classifiers is boosting algorithms, which are one of the pillars of classic machine
learning, often significantly improving the accuracy of a base
learning algorithm by creating an ensemble of multiple base
classifiers/hypotheses. More formally, consider binary classification over an
input domain $\cX$ and let $\cA$ be a base learning algorithm that on a
(possibly weighted) training sequence $S \in (\cX \times \{-1,1\})^*$, returns a hypothesis
$h_S \in \{-1,1\}^{\cX}$ better than guessing on the weighted training sequence. Boosting algorithms, like AdaBoost~\cite{adaboostyoavfreund}
then iteratively invokes $\cA$ on reweighed version of the training sequence
$S$ to produce hypotheses $h_1,\dots,h_t$. In each step, the training
sequence $S$ is weighed to put more emphasis on data points that
$h_1,\dots,h_{t-1}$ misclassified, forcing $h_t$ to focus on these
points. The obtained hypotheses are finally combined
via a weighted majority vote to obtain the \emph{voting classifier} $f(x) =
\sum_i \alpha_i h_i(x)$ with $\alpha_i >0$ for all $i$.
Much research has gone into understanding the impressive performance
of boosting algorithms and voting classifiers, with one particularly influential line of work
focusing on so-called \emph{margins}~\cite{Schapire19981651}. Given a
voting classifiers
$f(x)=\sum_i \alpha_i h_i(x)$, the margin of $f$ on a data point $(x,y) \in \cX
\times \{-1,1\}$ is then defined as
\[
  \margin(f,(x,y)):=\frac{y \sum_i \alpha_i h_i(x)}{\sum_i \alpha_i}.
\]
Observe that the margin is a number in $[-1,1]$. The margin is $1$ if
all hypotheses $h_i$ in $f$ agree and are correct, it is $-1$ if they
agree and are incorrect, and it is $0$ if there is a 50-50 weighted split in
their predictions for the label of $x$. The margins can thus be
thought of as a (signed) certainty in the prediction made by $f$.
Naturally $ \sign(f(x)) $ correctly predicts the label $ y $ if and only if $ \margin(f,(x,y))> 0 $, by taking $ \sign(0)=0. $     

Early boosting experiments~\cite{Schapire19981651} showed that the accuracy of voting
classifiers trained with AdaBoost often improves even when adding
more hypotheses $h_t$ to $f$ after the point where $f$ perfectly classifies the training
data $S$. As adding more hypotheses to $f$ results in a more
complicated model, this behavior was surprising.  Work
by~\cite{Schapire19981651} attribute these improvements in accuracy to improved
margins on the training data. In more detail, they proved
generalization bounds stating that large margins imply good
generalization performance. To give a flavor of these bounds, assume
for now that the hypotheses $h_i$ used for constructing a voting
classifier, all belong to a finite hypothesis class $\cH \subseteq
\{-1,1\}^{\cX}$. If we use $\ls^\gamma_{S}(f)$ to denote the
fraction of samples $(x,y)$ in $S$ for which $f$ has margin no more
than $\gamma$, then \cite{Schapire19981651} showed that for any data distribution
$\cD$ over $\cX \times \{-1,1\}$, any $0 < \delta < 1$ and $0 < \gamma
\leq 1$, it holds
with probability at least $1-\delta$ over a training sequence $\rS \sim
\cD^n$ that every voting classifier $f$ has
\begin{align}
 \ls_{\cD}(f):= \Pr_{(\rx, \ry) \sim \cD}[\ry f(\rx) \leq  0] =
  \ls^\gamma_\rS(f) + 
 O\left( \sqrt{\frac{\ln(|\cH|) \ln m}{\gamma^2 m} + \frac{\ln(1/\delta)}{m}}\right).\label{eq:simplegen}
\end{align}
As can be seen from this bound, large margins $\gamma > 0$ improve
generalization. For the case where all samples have margin at least
$\gamma$, i.e.\ $\ls_S^\gamma(f)=0$,~\cite{10.1162/089976699300016106} improved this to
\begin{align}
 \ls_{\cD}(f) = 
 O\left( \frac{\ln(|\cH|) \ln m}{\gamma^2 m} + \frac{\ln(1/\delta)}{m}\right).\label{eq:breimangen}
\end{align}
The current state-of-the-art margin generalization bounds nicely
interpolates between the two bounds above. Concretely,~\cite{GAO20131} proved the following generalization
of \cref{eq:simplegen} and~\cref{eq:breimangen} often referred to as the \emph{$k$'th margin
  bound} (for simplicity, we hide the dependency on $\delta$):
\begin{align}
  \ls_{\cD}(f) = \ls_{\rS}^\gamma(f) +
  O\left(\sqrt{\frac{\ls_{\rS}^\gamma(f) \ln(|\cH|) \ln m}{\gamma^2
  m}} +\frac{\ln(|\cH|) \ln m}{\gamma^2 m}\right). \label{eq:refinedgen}
\end{align}
Lower bounds show that this generalization bound is nearly
tight. In particular, the work~\cite{DBLP:conf/nips/GronlundKL20} showed that for any cardinality $N$, and parameters $1/m <\tau, \gamma < c$ for a sufficiently small constant
$c>0$, there is a data distribution $\cD$ and finite hypothesis class
$\cH$ with $|\cH|=N$, such that with constant probability over $\rS
\sim \cD^m$, there is a voting classifier $f$ over $\cH$ with
$\ls^\gamma_{\rS}(f) \leq \tau$ and
\begin{align}
  \ls_{\cD}(f) = \tau + \Omega\left(\sqrt{\frac{\tau \ln(N)
  \ln(1/\tau) }{\gamma^2 m}} +\frac{\ln(N) \ln m}{\gamma^2 m}\right). \label{eq:lower}
\end{align}
This matches the upper bound in \cref{eq:refinedgen} up to the gap
between $\sqrt{\ln(1/\tau)} \approx \sqrt{\ln(1/\ls^\gamma_{\rS}(f))}$ and
$\ln m$, improving by a $\sqrt{\ln(1/\tau)}$ factor over a previous
lower bound by~\cite{DBLP:conf/nips/GronlundKLMN19}. Note that we have simplified the lower bound slightly, as the
true statement would have $\ln m$ replaced by $\ln(\gamma^2
m/\ln N)$. A similar substitution of $\ln m$ by $\ln(\gamma^2
m/\ln(|\cH|))$ in the upper bound \cref{eq:refinedgen} also seems possible with minor modifications to the previous proof.

If we instead turn to the more general case of voting classifiers over
a possibly infinite hypothesis class $\cH$ of VC-dimension $d$, the
current state of affairs is less satisfying. Also in the work
by~\cite{GAO20131} introducing the $k$'th margin bound, they show that for any
data distribution $\cD$ and hypothesis class
$\cH$ of VC-dimension $d$, it holds with probability at least
$1-\delta$ over $\rS \sim \cD^m$ that every voting classifier $f$ over
$\cH$ satisfies
\begin{align}
\ls_{\cD}(f) = \ls_{\rS}^\gamma(f) +
  O\left(\sqrt{\ls_{\rS}^\gamma(f) \left(\frac{d \ln(m/d) \ln m}{\gamma^{2} 
  m} + \frac{\ln(1/\delta)}{m} \right)}+\frac{ d \ln(m/d) \ln m}{\gamma^{2} m}+\frac{\ln(1/\delta)}{m}\right). \label{eq:kth}
\end{align}
The only lower bound for finite VC-dimension is \cref{eq:lower} with
$\ln(N)$ replaced by $d$. The gap here is thus a logarithmic
factor and the generalization bound in \cref{eq:kth} has not seen any
improvements since.

\paragraph{New Margin Generalization Bounds.}
Our first main contribution is improved and almost optimal generalization bounds for
voting classifiers with large margins. Concretely, we show the
following theorem
\begin{theorem}
\label{thm:mainintro}[Informal statement of \cref{thm:finalmarginbound}]
For any hypothesis class $ \cH \subseteq \{  -1,1\}^{\cX}$ of VC-dimension
$ d $, distribution $ \cD $ over $ \cX\times \{  -1,1\}  $, failure
parameter $0<\delta<1$ and any constant $ 0<\eps<1 $, it holds with
probability at least $ 1-\delta $ over $ \rS\sim \cD^{m} $ that for
any margin $ 0<\gamma\leq 1 $ and any voting classifier $f$ over
$\cH$, we have
    \begin{align*}
     \ls_{\cD}(f)=  
     \ls_{\rS}^{\gamma}(f)
     +O\left(\sqrt{ \ls_{\rS}^{\gamma}(f) \left(\frac{d\func{\left(\frac{\gamma^{2}m }{d} \right)}}{\gamma^{2}m}
     +\frac{\ln{\left(\frac{1}{\delta} \right)}}{m}\right)}  + \frac{d\func{\left(\frac{\gamma^{2}m }{d} \right)}}{\gamma^{2}m}+\frac{\ln{\left(\frac{1}{\delta} \right)}}{m}\right),
    \end{align*}
where $\func(x) = \ln(x) \ln^2(\ln  x)$.
  \end{theorem}
Our theorem improves over \cref{eq:kth} by nearly a logarithmic
factor and the gap between our upper bound and the lower bound
in~\cref{eq:lower} is essentially $\ln(\ln(\gamma^2 m/d))$
times the ratio between $\sqrt{\ln(1/\ls^\gamma_\rS(f))}$ and
$\sqrt{\ln(\gamma^2 m/d)}$. Furthermore, all logarithmic factors are now $\ln(\gamma^2
m/d)$ instead of $\ln(m/d)$ and $\ln m$. While the improvement inside
the $\ln$'s might seem 
minor, this has crucial implications for the development
of a new boosting algorithm explained later. Furthermore, and unlike in the
case of finite $\cH$ (see discussion after~\cref{eq:lower}), there does not seem to be a way of tweaking the
proof of the previous bound in~\cref{eq:kth} to improve the factors $\ln(m/d)$ and $\ln m$
to $\ln(\gamma^2 m/d)$. 

\paragraph{New Boosting Results.}
One of the prime motivations for studying generalization bounds for large margin voting
classifiers, is their application to boosting algorithms. When
studying boosting theoretically, we typically use the framework of
\emph{weak to strong} learning
by~\cite{kearns1988learning,kearns1994learning}. Let $c^{*} \in
\{-1,1\}^\cX$ be an unknown target concept assigning labels $c^{*}(x)$ to
samples $x \in \cX$. For a distribution $\cD$ over $\cX$, let $\cD_{c^{*}}$
be the distribution over $\cX \times \{-1,1\}$ obtained by drawing a
sample $\rx \sim \cD$ and returning the pair $(\rx, c^{*}(\rx))$.

A $\gamma$-weak
learner $\cW$, is a learning algorithm that for any distribution $\cD$
over $\cX$, when
given $m_0$ i.i.d.\ samples $(\rx_i,c^{*}(\rx_i)) \sim \cD_{c^{*}}$, $\cW$ produces with probability at least
$1-\delta_0$ a hypothesis $h$ with
$\ls_{\cD_{c^{*}}}(h) = \Pr_{(\rx,c^{*}(\rx)) \sim \cD_{c^{*}}}[h(\rx) \neq c^{*}(\rx)] \leq 1/2-\gamma$. Here $m_0$ and $\delta_0$ are
constants. A strong learner in contrast, is a learning
algorithm that for any distribution $\cD$ over $\cX$, when given $m(\eps,
\delta)$ i.i.d.\ samples from $\cD_{c^{*}}$, it produces with probability at least
$1-\delta$ a hypothesis with $\ls_{\cD_{c^{*}}}(h) \leq \eps$. A strong learner
thus obtains arbitrarily high accuracy when given enough training
data.
AdaBoost~\cite{adaboostyoavfreund} is the most famous algorithm for constructing a strong
learner from a weak learner. Concretely, it can be shown that if
AdaBoost is run for $O(\gamma^{-2} \ln m)$ iterations, then it
produces a voting classifier $f$ with margin $\Omega(\gamma)$ on all
samples in a given training sequence $S$ with $|S|=m$~\cite{boostingbook} [Theorem 5.8]. If the weak
learner/base learning algorithm always returns hypotheses from a
hypothesis class $\cH$ of VC-dimension $d$, this allows us to use our
new generalization bound in~\cref{thm:mainintro} to conclude
\begin{corollary}
  \label{cor:ada}
For any $\gamma$-weak learner $\cW$ using a hypothesis class $ \cH \subseteq \{  -1,1\}^{\cX}$ of VC-dimension
$ d $, distribution $ \cD $ over $ \cX$, target
concept $c^{*}$, failure
parameter $0<\delta<1$ and any constant $ 0<\eps<1 $, it holds with
probability at least $1-\delta$ over $\rS \sim \cD_{c^{*}}^m$, that the voting
classifier $f$ produced by AdaBoost on $\rS$ with weak learner $\cW$ has
\[
  \ls_{\cD_{c^{*}}}(f) = O\left(\frac{d \func(\gamma^2m/d)}{\gamma^2 m} +
  \frac{\ln(1/\delta)}{m}\right),
\]
where $\func(x) = \ln(x)\ln^2(\ln x)$.
\end{corollary}
The previous best known generalization bound for AdaBoost followed
from the margin generalization bound \cref{eq:kth} and was
$\ls_{\cD_{c^{*}}}(f) = O\left(\frac{d \ln(m/d)\ln m}{\gamma^2 m} +
  \frac{\ln(1/\delta)}{m}\right)$.
Moreover, our new bound is tight up to a $\ln^2(\ln(\gamma^2 m/d))$ factor as
demonstrated by a lower bound of~\cite{adaboostnotoptimal} showing that for  
$ \gamma < c$ for sufficiently small constant $c>0$ and VC-dimension $d = \Omega(\ln(1/\gamma))$, and sample size $ d\gamma^{-2}/16\leq m\leq  \exp{\left( \exp{\left(d \right)}  \right)},$ there is a data
distribution $\cD$, a weak learner $\cW$ using a hypothesis class of
VC-dimension $d$, and  a concept $c^{*}$, such that AdaBoost run with
$\cW$ has
$\ls_{\cD_{c^{*}}}(f) = \Omega\left(\frac{d \ln(\gamma^2m/d)}{\gamma^2 m}\right)$,
 with constant probability over a training sequence $\rS \sim
 \cD_{c^{*}}^m$. 

 In addition to improving our understanding of AdaBoost, our new
 generalization bound for voting classifiers also allows us to design
 an improved weak to strong learner with an optimal in-expectation error. In the
 work~\cite{larsen2022optimalweakstronglearning}, it was shown that the optimal generalization
 error of any weak to strong learning algorithm with access to a
 $\gamma$-weak learner using a hypothesis class of VC-dimension $d$ is
 \begin{align}
   \ls_{\cD_{c^{*}}}(f) = \Theta\left(\frac{d}{\gamma^2m} + \frac{\ln(1/\delta)}{m}\right).\label{eq:optstrong}
 \end{align}
 In light of the lower bound above for AdaBoost, this implies that
 AdaBoost is not an optimal weak to strong learner. However, several optimal weak to strong
 learning algorithms have been developed. In~\cite{larsen2022optimalweakstronglearning}, the authors gave
 the first such algorithm. This algorithm uses the sub-sampling idea
 of~\cite{Hannekeoptimal} from optimal realizable PAC learning and runs AdaBoost on
 $m^{\lg_4 3}$ many sub-samples $S_i \subset S$ of the training
 data. It combines the produced voting classifiers by taking a
 majority vote among their predictions, i.e.\ a majority-of-majorities. Since each $S_i$ has $|S_i|= \Omega(m)$, this slows down
 AdaBoost by a factor $m^{\lg_4 3} \approx m^{0.79}$. Later work
 by~\cite{baggingoptimal} showed that running Bagging~\cite{Breiman1996BaggingP} to draw
 $O(\ln(m/\delta))$ many random sub-samples
 $S_i$ from $S$ and running AdaBoost on each also results in an
 optimal generalization error matching \cref{eq:optstrong}, thus
 reducing the computational overhead to a logarithmic factor. Finally,
 a recent work by~\cite{manyfacesofoptimalweaktostronglearning} built on a Majority-of-3 result in
 realizable PAC learning~\cite{majorityofthree} to show that simply partitioning a training
 sequence into $5$ disjoint pieces of $m/5$ samples each, and
 outputting a majority vote among voting classifiers trained with
 AdaBoost on each sub-sample, results in an optimal in-expectation
 error of $\e_{\rS}[\ls_{\cD_{c^{*}}}(f)] = O(d/(\gamma^2 m))$.

 Our new generalization bound in~\cref{thm:mainintro} allows us to
 improve the Majority-of-5 algorithm to a more natural Majority-of-3 of AdaBoost (we follow the AdaBoost implementation on page 5 \cite{boostingbook}). We here assume that Adaboost has access to a deterministic empirical $ \gamma $-weak learner, where an empirical $ \gamma $-weak learner, takes as input a sample $ \rS\sim \cD_{c^{*}}^{n} $ and a distribution $ \cD' $ over $ \rS $ and returns a hypothesis $ h $ such that $ \ls_{\cD'}(h)\leq 1/2-\gamma.$ The existence of such a empirical $ \gamma $-weak learner is given by the prior definition of a $ \gamma $-weak learner. 

 \begin{corollary}
 \label{cor:maj3intro}[Follows from \cref{thm:mainmajoritythree}]
For any $\gamma$-weak learner $\cW$ using a hypothesis class $ \cH \subseteq \{  -1,1\}^{\cX}$ of VC-dimension
$ d $, distribution $ \cD $ over $ \cX$, and concept $c^{*}$, it holds that the voting
classifiers $f_{\rS_1},f_{\rS_2},f_{\rS_3}$ produced by AdaBoost on
i.i.d.\ training sequences $\rS_1,\rS_2,\rS_3 \sim \cD_{c^{*}}^m$ with
weak learner $\cW$ satisfy
    \begin{align*}
    \e_{\rS_{1},\rS_{2},\rS_{3}\sim \cD_{c^{*}}^{m}}\left[\ls_{\cD_{c^{*}}}(\maj(f_{\rS_{1}},f_{\rS_{2}},f_{\rS_{3}}))\right] =O\left(\frac{d}{\gamma^{2}m}\right).
    \end{align*}
  \end{corollary}
It is in this result that it is critical that our generalization bound
in~\cref{thm:mainintro} has $\ln(\gamma^2 m/d)$ factors rather than
$\ln(m/d)$ or $\ln(m)$ factors. We elaborate on this in~\cref{sec:sketchmaj3}.
\section{Proof Overviews and Notation}
In this section, we first describe the ideas going into our improved generalization bound for voting classifiers, stated in~\cref{thm:mainintro}. We then proceed to give an overview of our proof that Majority-of-3 AdaBoosts gives an optimal in-expectation error for weak to strong learning as in~\cref{cor:maj3intro}. Along the way, we introduce notation that we will use in our proofs.

\subsection{New Margin Generalization Bounds}
Recall that our goal is to establish~\cref{thm:mainintro}, showing that with probability at least $1-\delta$ over $ \rS \sim \cD^m$, it holds for all $ \gamma\in\left[0,1\right] $ and voting classifiers $f$ over a hypothesis class $\cH$ of VC-dimension $d$ that:
\begin{align}
 \ls_{\cD}(f)= \ls^{\gamma}_{\rS}(f)+O\left(\sqrt{\ls^{\gamma}_{\rS}(f)\left(\frac{d\func{\left(\frac{\gamma^{2}m}{d} \right)}}{\gamma^{2}m}+\frac{\ln{\left(\frac{1}{\delta}\right)}}{m} \right)}+\frac{d\func{\left(\frac{\gamma^{2}m}{d} \right)}}{\gamma^{2}m}+\frac{\ln{\left(\frac{1}{\delta}\right)}}{m} \right),\label{eq:overviewgoal}
\end{align}
where $\func(x) = \ln(x)\ln^2(\ln x)$.
Let us first introduce a more convenient way of representing voting classifiers. Recall that voting classifiers $f$ are of the form $f(x) = \sum_i \alpha_i h_i(x)$ with all $\alpha_i > 0$. Furthermore, the margin on a training example $(x,y)$ is defined as $y \sum_i \alpha_i h_i(x)/\sum_j \alpha_j$. To avoid the tedious normalization by $\sum_j \alpha_j$, we henceforth assume all voting classifiers have $\sum_i \alpha_i=1$. In this way, we have that $\sign(f(x)) \neq y$ if and only if $y f(x) \leq 0$ (we define $\sign(0)=0$). Furthermore, the margin is no more than $\gamma$ if and only if $y f(x) \leq \gamma$. We hence define $\dlh$ as the set of all convex combinations $\sum_i \alpha_i h_i$ for $h_i \in \cH$ (i.e. $\sum_i \alpha_i =1$ and $ \alpha_{i}>0 $) and refer to $\dlh$ as the set of voting classifiers over $\cH$. With this notation, we have $\ls^\gamma_\cD(f) = \Pr_{(\rx, \ry) \sim \cD}[\ry f(\rx) \leq \gamma]$, $\ls^\gamma_S(f)$ is the fraction of training examples $(x,y) \in S$ with $yf(x) \leq \gamma.$ For $ \gamma=0 $ we will use  $\ls_\cD(f)= \ls^0_\cD(f) =\Pr_{(\rx, \ry) \sim \cD}[\ry f(\rx) \leq 0]=\Pr_{(\rx, \ry) \sim \cD}[\sign(f(\rx))\not=\ry]$.

\paragraph{Partitioning into Intervals.}
To establish~\cref{eq:overviewgoal}, we first simplify the task by partitioning the range of $\gamma$ and $\ls_\rS^\gamma(f)$ into small intervals $[\gamma_0^i,\gamma_1^i]$ and $[\tau_0^{i,j},\tau_1^{i,j}]$, respectively. For each interval $[\gamma_0,\gamma_1]=[\gamma_0^i,\gamma_1^i]$ and $[\tau_0,\tau_1] = [\tau_0^{i,j},\tau_1^{i,j}]$, we show that for any $\delta$, with probability at least $1-\delta$ we have for every $f \in \dlh$ and $\gamma \in [\gamma_0,\gamma_1]$ that either: $\ls^\gamma_{\rS}(f) \notin [\tau_0,\tau_1]$ or
\begin{align}\label{eq:proofsketch2}
 \ls_{\cD}(f)= \tau_{1}+O\left(\sqrt{\tau_{1}\left(\frac{d\func{\left(\frac{m \gamma_{0}^{2}}{d} \right)}}{\gamma_{0}^{2}m}+\frac{\ln{\left(\frac{1}{\delta}\right)}}{m} \right)}+\left(\frac{d\func{\left(\frac{m\gamma_{0}^{2}}{d} \right)}}{\gamma_{0}^{2}m}+\frac{\ln{\left(\frac{1}{\delta}\right)}}{m} \right) \right).
\end{align}
We can then union bound over all intervals, choosing appropriate values $\delta_{i,j} < \delta$, to conclude that with probability $1-\delta$, the guarantee~\cref{eq:proofsketch2} holds simultaneously for all intervals. If we choose the length of the intervals small enough, all $\gamma \in [\gamma_0, \gamma_1]$ and $\ls^\gamma_{\rS}(f) \in [\tau_0,\tau_1]$ are sufficiently close that we may substitute all occurrences of $\gamma_0$ and $\tau_1$ in~\cref{eq:proofsketch2} by $\gamma$ and $\ls^\gamma_{\rS}(f)$. Such a partitioning is standard in proofs of margin bounds, although we have to be a little careful in defining the intervals small enough. Having done so, this recovers~\cref{eq:overviewgoal}.

\paragraph{Ghost Set.}
Thus we focus on showing the claim in~\cref{eq:proofsketch2}. As in many previous proofs of generalization bounds, we first seek to discretize the infinite hypothesis class $\dlh$ and then apply a union bound over a finite set of events/hypotheses. In our proof, we will construct a $(\gamma_0/2)$ $\ell_\infty$-covering $N$ of $\dlh$. Ideally, such a covering would contain for every $f \in \dlh$, a function $f' \in N$ such that $|f(x)-f'(x)| \leq \gamma_0/2$ all $x \in \cX$. Unfortunately, there might not be a finite such $N$ when requiring $|f(x)-f'(x)| \leq \gamma_0/2$ for all $x$ in the full input domain $\cX$. We thus start by introducing a \emph{ghost set} $\rS' \sim \cD^m$ to replace all references to $\ls_{\cD}(f)$ (which depends on the full domain $\cX$) by $\ls_{\rS'}(f)$ (which depends only on $\rS'$). Using standard arguments relating $\ls_{\cD}(f)$ to $\ls_{\rS'}(f)$, we show that~\cref{eq:proofsketch2} follows if we can show that with probability $1-\delta$ over the pair $(\rS,\rS')$, it holds for every $f \in \dlh$ and $\gamma \in [\gamma_0,\gamma_1]$ that either $\ls_{\rS}^\gamma(f) \notin [\tau_0, \tau_1]$ or
\begin{align}\label{eq:goalWithghost}
 \ls_{\rS'}(f)=\tau_{1}+O\left(\sqrt{\tau_{1}\left(\frac{d\func{\left(\frac{m \gamma_{0}^{2}}{d} \right)}}{\gamma_{0}^{2}m}+\frac{\ln{\left(\frac{1}{\delta}\right)}}{m} \right)}+\left(\frac{d\func{\left(\frac{m\gamma_{0}^{2}}{d} \right)}}{\gamma_{0}^{2}m}+\frac{\ln{\left(\frac{1}{\delta}\right)}}{m} \right) \right).
\end{align}
Observe that we substituted $\ls_{\cD}(f)$ by $\ls_{\rS'}(f)$ compared to~\cref{eq:proofsketch2}.

\paragraph{Covering.}
To establish~\cref{eq:goalWithghost}, consider first a fixed $ f\in \dlh $. Since $\rS$ and $\rS'$ are i.i.d.\ samples from $\cD^m$, we have that $\ls^\gamma_{\rS}(f)$ and $\ls_{\rS'}(f)$ are strongly concentrated around their means $\ls^\gamma_\cD(f)$ and $\ls_\cD(f)$. Moreover, since $\ls^\gamma_\cD(f) \geq \ls_\cD(f)$, it is highly unlikely that $\ls_{\rS'}(f)$ is significantly larger than $\ls_\rS(f)$. This is precisely what we need to establish~\cref{eq:goalWithghost}. Concretely, we need to show that there is no $f$ with $\ls^\gamma_\rS(f) \in [\tau_0, \tau_1]$ such that $\ls_{\rS'}(f)$ is much larger than $\tau_1$. Since this event is unlikely for a fixed $f$, we now introduce a discretization of $\dlh$ that would preserve any large gap between $\ls^\gamma_\rS(f)$ and $\ls_{\rS'}(f)$.

To this end, we discretize $\dlh$ on $ \rX=\rS\cup\rS' $  via a $\gamma_0/2$ $\ell_\infty$-covering $ \Net $, i.e., for any $f \in \dlh$, there is an $f' \in \Net$ with $|f(x)-f'(x)| \leq \gamma_0/2$ for all $x\in \rX$. Now observe that whenever $yf(x) > \gamma_0$, we also have $yf'(x) > \gamma_0/2$. Thus, for any $\gamma \in [\gamma_0,\gamma_1]$, we have $\ls_{\rS}^\gamma(f) \geq \ls_{\rS}^{\gamma_0/2}(f')$. Similarly, we have for $yf(x) \leq 0$ that $yf'(x) \leq \gamma_0/2$, and thus $\ls_{\rS'}(f) \leq \ls_{\rS'}^{\gamma_0/2}(f')$. Therefore, a $\gamma_0/2$ $\ell_\infty$-covering $ \Net $ preserves the imbalance between $ \ls_{\rS}^{\gamma}(f) $ and $ \ls_{\rS'}(f)$ via $ \ls_{\rS}^{\gamma_0/2}(f') $ and $ \ls_{\rS'}^{\gamma_0/2}(f').$   
 
To construct a $\gamma_0/2$ $\ell_\infty$-covering $ \Net $ of $ \rX $ and union bound over it, we need the point set $ \rX $ to be fixed - however we still want to be able to show that an imbalance between $ \ls_{\rS}^{\gamma_0/2}(f') $ and $ \ls_{\rS'}^{\gamma_0/2}(f')$ for some $ f'\in \Net$ is highly unlikely. As in previous works, we employ the following way of viewing the sampling of $ \rS $ and $ \rS'$. First, we draw $ \rX\sim \cD^{2m} $, consisting of $ 2m $ i.i.d.\ training examples from $ \cD $, and then let $ \rS $ be $ m $ points drawn without replacement from $ \rX,$ and $ \rS' $ be the remaining $ m $  points of $ \rX,$ i.e., $ \rS'=\rX\backslash \rS'.$ Taking this viewpoint of drawing $ \rS $ and $ \rS' $ allows us to fix the realization $ X $ of the points in $ \rX,$ while still having which training examples ending up in $ \rS $ and $ \rS' $ being random. This still allows us to argue that an imbalance between $ \ls_{\rS}^{\gamma_0/2}(f') $ and $ \ls_{\rS'}^{\gamma_0/2}(f')$ for some $ f'\in \Net$ is unlikely.
              
Thus, we now consider an arbitrary but fixed realization $ X $ of $ \rX, $ and let $ \Net $ be a $\gamma_0/2$ $\ell_\infty$-covering of $ X. $ By the above arguments above, if we can show for any $0 < \delta < 1$ and an arbitrary $f\in \Net $, it holds with probability at least $1-\delta$ over the random partitioning of $X$ into $\rS, \rS'$ that either $\ls_{\rS}^{\gamma_0/2}(f) > \tau_1$ or
 \begin{align}\label{eq:proofsketch6}
  \ls^{\gamma_{0}/2}_{\rS'}(f) = \tau_{1}+O\left(\sqrt{\tau_{1}\frac{\ln{\left(\frac{1}{\delta}\right)}}{m} }+\frac{\ln{\left(\frac{1}{\delta}\right)}}{m} \right),
 \end{align}
 then we can union bound over all $f \in \Net$, with $\delta$ rescaled to $\delta/|\Net|$, to conclude that with probability $1-\delta$ it holds for all $f \in \Net$ that either $\ls_{\rS}^{\gamma_0/2}(f) > \tau_1$ or
 \begin{align*}
  \ls^{\gamma_{0}/2}_{\rS'}(f) = \tau_{1}+O\left(\sqrt{\tau_{1}\frac{\ln{\left(\frac{|\Net|}{\delta}\right)}}{m} }+\frac{\ln{\left(\frac{|\Net|}{\delta}\right)}}{m} \right).
 \end{align*}
Giving an appropriate upper bound on $|\Net|$ will then imply~\cref{eq:goalWithghost}
 
Now, to argue~\cref{eq:proofsketch6} for a fixed $ f \in \Net$, we want to show that the event $ \ls_{\rS}^{\gamma_{0}/2}(f)\leq \tau_{1} $ and $ \ls_{\rS'}^{\gamma_{0}/2}(f) = \tau_{1}+\Omega(\sqrt{\tau_{1}\ln{\left(1/\delta \right)}/m}+\ln{\left(1/\delta \right)}/m) $ happens with probability at most $ \delta$. Let $\mu$ denote the fraction of mistakes $ f $ makes on $ X $ and observe that $\mu= (\ls_{\rS}^{\gamma_{0}/2}(f)+\ls_{\rS'}^{\gamma_{0}/2}(f))/2$. We notice that $ \mu $ has to be at least $ \ls_{\rS'}^{\gamma_{0}/2}(f)/2 = (\tau_{1}+\Omega(\sqrt{\tau_{1}\ln{\left(1/\delta \right)}/m}+\ln{\left(1/\delta \right)}/m))/2 $ for the event to occur. Since $ \mu $ is $ \Omega(\ln{\left( 1/\delta\right)}/m) $ and $ \ls_{\rS}^{\gamma_{0}/2}(f) $ has expectation equal to $ \mu,$ it follows by an invocation of a Chernoff bound (without replacement) that with probability at least $ 1-\delta $ over $ \rS $ (drawn from $ X $)  that
 \begin{align*}
   \ls_{\rS}^{\gamma_{0}/2}(f)&\geq \left(1-\sqrt{\frac{2\ln{\left(1/\delta \right)}}{\mu m}}\right)\mu\\
   &=(\ls_{\rS}^{\gamma_{0}/2}(f)+\ls_{\rS'}^{\gamma_{0}/2}(f))/2-\sqrt{\frac{2\ln{\left(1/\delta \right)}(\ls_{\rS}^{\gamma_{0}/2}(f)+\ls_{\rS'}^{\gamma_{0}/2}(f))/2}{m}}, 
 \end{align*}    
where doing some rearrangements implies the following inequality
 \begin{align*}
  \frac{\ls_{\rS}^{\gamma_{0}/2}(f)}{2}+\sqrt{\frac{\ls_{\rS}^{\gamma_{0}/2}(f)\ln{\left(1/\delta \right)}}{m}}\geq \frac{\ls_{\rS'}^{\gamma_{0}/2}(f)}{2}-\sqrt{\frac{\ls_{\rS'}^{\gamma_{0}/2}(f)\ln{\left(1/\delta \right)}}{m}}.
 \end{align*}
We notice that the above inequality is implying that $ \ls_{\rS'}^{\gamma_{0}/2}(f) $ cannot be too large compared to $ \ls_{\rS}^{\gamma_{0}/2}(f).$ Specifically the inequality implies that for $ \ls_{\rS}(f)\leq \tau_{1} $, we must have $ \ls_{\rS'}^{\gamma_{0}/2}(f)= \tau_{1}+O(\sqrt{(\tau_{1}\ln{\left(1/\delta \right)})/m}+\ln{\left(1/\delta \right)}/m)$ as desired. Let us finally remark that applying Hoeffding's inequality would be insufficient to obtain our bounds in that we crucially exploit that Chernoff (or Bernstein's) gives bounds relative to the mean $\mu$.

\paragraph{Clipping.}
While the above argument gives the correct type of bound  $\tau_{1} +\sqrt{\tau_{1}\ln{\left(|N|/\delta \right)}/m} +\ln{\left(|N|/\delta \right)}/m$, the size of the above suggested cover $ N $ turns out to be too large.
The intuitive reason is that the functions $f \in \dlh$ take values in the range $[-1,1]$, whereas we only really care about the values being larger than $\gamma$ or smaller than $-\gamma$ in the losses $\ls_{\rS}^\gamma(f)$ and $\ls_{\rS'}^\gamma(f)$. Constructing a cover for the full range $[-1,1]$ thus leads to a larger cover size than necessary and hence is too costly for a union bound.
Our idea to remedy this, is to \emph{clip} the voting classifiers in $\dlh$. For this let $ \gamma>0 $, and $ f $ be a function from $ \cX $ into $ [-1,1],$ we then define  $ f_{\left\lceil\gamma\right\rceil} $ as the function from $ \cX\rightarrow [-1,1] $ given by
\begin{align}
    f_{\left\lceil\gamma\right\rceil}(x)=\begin{cases}
        \gamma &\text{ if }f(x)> \gamma
        \\
        -\gamma &\text{ if } f(x)<- \gamma
        \\
        f(x) &\text{ else }
    \end{cases},
\end{align}         
and $ \dlh_{\left\lceil\gamma \right\rceil}=\{ f_{\left\lceil\gamma\right\rceil}:f\in \dlh \},$ i.e. the functions in $ \dlh $ capped to respectively $ -\gamma $ and $ \gamma $ if it goes below or above $ -\gamma $ or $ \gamma $. We will show that $\dlh_{\left\lceil\gamma_1 \right\rceil}$ has a small $\gamma_0/2$-cover $ \Net $ of cardinality just
\begin{align}
  \cN_{\infty}(X,\dlh_{\left\lceil\gamma_1 \right\rceil},\gamma_0/2)=\exp{\left(O\left( \frac{d}{\gamma_0^{2}}\func{\left(\frac{m\gamma_0^{2}}{d} \right)} \right)\right)}, \label{eq:coversize}
\end{align} 
where the notation $ \cN_{\infty}(\cdot,\cdot,\cdot) $  for a point set $ X\subseteq \cX$   function class $ \cF\subseteq \mathbb{R}^{\cX}$ and precision parameter $ \alpha $ means the smallest size, $ \cN_{\infty}(X,\cF,\alpha),$  of an $ \alpha $ $ \ell_{\infty} $-covering of $ \cF $, $ \Net,$  on $ X $. 
\paragraph{Relating Covering Number and Fat Shattering.}
We finally need to bound the covering number as in~\cref{eq:coversize}, where a key part of the argument is that we use $ \dlh_{\left\lceil\gamma_1 \right\rceil} $ instead of $\dlh $.  With the goal of establishing this bound on the cover size, we use a result by \cite{RudelsonVershynin} relating the covering number to fat shattering. Let us first recall the definition of fat shattering.

For at point set $\{  x_{1},\ldots,x_d\} $ of size $ d $ and a level parameter $ \beta>0 $, we say that a function class $ \cF $ $ \beta $-shatters   $ \{  x_{1},\ldots,x_d\} $ if there exists levels $ r_1,\ldots,r_{d} $ such that for any $ b\in \{  -1,1\}^{d} $, we have that there exists $ f\in \cF $ such that 
    \begin{align*}
        f(x_i)&\leq r_{i}-\beta  \quad \text{ if  } b_{i}=-1
        \\
        f(x_i)&\geq r_{i}+\beta  \quad \text{ if  } b_{i}=1,
    \end{align*}
    that is, the function class is rich enough to be $ \beta $  above or below the levels $ r_{1},\ldots,r_{d} $ on the point set $ \{  x_{1},\ldots,x_d\}$.  
    For a function class $ \cF $ and level $ \beta>0 $  we define $ \fat_{\beta}(\cF) $  as the largest number $ d $,  such that there exists a point set $ x_{1},\ldots,x_{d}   $ of size $ d $, which is $ \beta $-shattered by $ \cF $.

    With the definition of $\fat_\beta(\cF)$ in place,~\cite{RudelsonVershynin}  [Theorem 4.4] says that for any $ 0<\alpha <1/2$, any $0 < \eps < 1$, any function class $ \cF $ with $ \fat_{c\alpha\eps} $-dimension $ d_{c\alpha\eps} $, for a constant $ c>0 $, and any point set $ X\subseteq \cX $, with $ |X|=m $, such that $ \sum_{x\in X}|f(x)|/m \leq 1$ for any $ f\in \cF $, it holds that $ \ln{\left(\cN_{\infty}(X,\cF,\alpha) \right)}=O\left(d_{c\alpha\eps}\ln{\left(\frac{m}{d_{c\alpha\eps}\alpha} \right)}\ln^{\eps}{\left(\frac{m}{d_{c\alpha\eps}} \right)}\right)$. 
    The above bound looks quite different from \cref{eq:coversize}, but we will later choose the variable $ \eps $ in an appropriate way and recover \cref{eq:coversize}. 

A first naive approach would be to invoke the result of~\cite{RudelsonVershynin} directly on $\dlh_{\left\lceil\gamma_1 \right\rceil}$ (or even the unclipped $\dlh$), to conclude that
\begin{align}
\ln{\left(\cN_{\infty}(X,\dlh_{\left\lceil\gamma_1 \right\rceil},\gamma_0/2) \right)}=O\left(d_{c\gamma_0\eps/2}\ln\left(\frac{m}{d_{c\gamma_0\eps/2} \gamma_0} \right)\ln^{\eps}\left({\left(\frac{m}{d_{c\gamma_0\eps/2}} \right)}\right) \right), \label{eq:naivecover}
\end{align}
We will later show that $d_{c\gamma_0\eps/2} = O(d/ (\gamma_0\eps)^{-2})$ for $\dlh_{\left\lceil\gamma_1 \right\rceil}$, with $ d $ being the VC-dimension of $ \cH $. Inserting this in the above and considering $ \eps $ as a constant we see this  fails to recover our claimed covering number in~\cref{eq:coversize}. In particular, if considering $ \eps $ as a constant, the $\ln(m/(d_{c\gamma_0\eps/2}  \gamma_0))$ factor would become $\ln(\gamma_0 m/d)$ rather than the claimed $\ln(\gamma_0^2 m/d)$. Again, this difference turns out to be crucial for our Majority-of-3 algorithm as we will argue later.

To remedy this, we exploit the clipping. We observe that for $ f\in \dlh_{\left\lceil\gamma_1 \right\rceil} $, we have that $ \sum_{x\in X}|f(x)|/m \leq \gamma_1$. We may thus invoke the result of~\cite{RudelsonVershynin} on the scaled function class $ \gamma_1^{-1} \cdot\dlh_{\left\lceil\gamma_1 \right\rceil} =\{ f\mid \exists f'\in \dlh_{\left\lceil\gamma_1 \right\rceil}, f=\gamma_1^{-1} f'  \} $, i.e.\ the functions in $ \dlh_{\left\lceil\gamma_1 \right\rceil} $ scaled by $ \gamma_1^{-1} $ to get that
\[
  \ln{\left(\cN_{\infty}(X,\gamma_1^{-1} \cdot\dlh_{\left\lceil\gamma_1 \right\rceil},\gamma_0/(2\gamma_1)) \right)}=O\left(d_{c\gamma_0\eps/(2\gamma_1)}\ln^{1+\eps}{\left(\frac{m}{d_{c\gamma_0\eps/(2\gamma_1)}} \right)}\right),
\]
where $ d_{c\gamma_0\eps/(2\gamma_1)} $ is the $ \fat_{c\gamma_0\eps/(2\gamma_1)}$-dimension of $ \gamma_1^{-1} \cdot\dlh_{\left\lceil\gamma_1 \right\rceil}$. Picking a minimal $\gamma_0/(2 \gamma_1)$ covering $\Net$ for $\gamma_1^{-1} \cdot\dlh_{\left\lceil\gamma_1 \right\rceil}$ and downscaling all functions in $\Net$ by $\gamma_1$ results in $\gamma_1 \Net$ being a $\gamma_0/2$ covering for $\dlh_{\left\lceil\gamma_1 \right\rceil}$ as required. We have thus exploited the clipping to show that
\[
  \ln{\left(\cN_{\infty}(X,\dlh_{\left\lceil\gamma_1 \right\rceil},\gamma_0/2) \right)}=O\left(d_{c\gamma_0\eps/(2\gamma_1)}\ln^{1+\eps}{\left(\frac{m}{d_{c\gamma_0\eps/(2\gamma_1)}} \right)}\right).
\]
 Now the  $ \fat_{c\gamma_0\eps/(2 \gamma_1)} $-dimension of $ \gamma_1^{-1} \cdot\dlh_{\left\lceil\gamma_1 \right\rceil}$ is, due to the scale invariance of $ \fat $-dimension, the same as the  $ \fat_{c\gamma_0\eps/2} $-dimension of $ \dlh_{\left\lceil\gamma_1 \right\rceil}$. We have thus improved~\cref{eq:naivecover} to
\begin{align*}
\ln{\left(\cN_{\infty}(X,\dlh_{\left\lceil\gamma_1 \right\rceil},\gamma_0/2) \right)}=O\left(d_{c\gamma_0\eps/2}\ln^{1+\eps}\left(\frac{m}{d_{c\gamma_0\eps/2} } \right) \right), 
\end{align*}
Inserting the claimed bound of $d_{c\gamma_0\eps/2} = O(d (\gamma_0\eps)^{-2})$ and setting $ \eps=1/\ln(\ln(m\gamma_{0}^{2}/d)) $  gives
\begin{align*}
  \ln{\left(\cN_{\infty}(X,\dlh_{\left\lceil\gamma_1 \right\rceil},\gamma_0/2) \right)}=O\left(\frac{d}{\gamma_{0}^{2}\eps^{2}}\ln^{1+\eps}\left(\frac{m\gamma_{0}^{2}\eps^{2}}{d} \right) \right)=O\left(\frac{d}{\gamma_{0}^{2}}\ln^{2}{\left(\ln{\left( \frac{m\gamma_{0}^{2}}{d}\right)} \right)}\ln\left(\frac{m\gamma_{0}^{2}}{d} \right) \right), 
  \end{align*}
where in the last inequality have used that $ \exp(\eps\ln{\left(\ln{\left(m\gamma_{0}^{2}/d \right)} \right)})=O(1).$ Since $ \func(x)=\ln^{2}{\left(\ln{\left(x \right)} \right)}\ln{\left(x \right)} $ the above establishes~\cref{eq:coversize} and completes our bound on the covering number and. All that remains is thus to argue that $d_{c\gamma_0\eps/2}= O(d (\gamma_0\eps)^{-2})$.

\paragraph{Bounding Fat Shattering Dimension.}
To bound $d_{c\gamma_0\eps/2}$, we use an argument similar to the proof of \cite{larsen2022optimalweakstronglearning} [Lemma 9]. Assume $ \dlh_{\left\lceil\gamma_1 \right\rceil}$ $ c\gamma_0\eps/2$-shatters a set of $n$ points $ x_1,\ldots,x_{n} $, with witness $ r_{1},\ldots,r_{n}\in [-\gamma_1,\gamma_1] $. By definition of shattering, we then have that for any $  b\in \{ -1, 1\}^{n}$, there exists $ f\in  \dlh_{\left\lceil\gamma_1 \right\rceil}$ such that $ b_{i}(f(x_{i})-r_{i})\geq c\gamma_0\eps/2$ for all $ i=1,\ldots,n$.

We next observe that since $ f\in  \dlh_{\left\lceil\gamma_1 \right\rceil}$ is equal to $ \min(\max(f',-\gamma_1),\gamma_1) $  for $ f'\in \dlh $, we have that $ f' $ also satisfies $ b_{i}(f'(x_{i})-r_{i})\geq c\gamma_0\eps/2,$ implying that $ \dlh$ also $c\gamma_0\eps/2$-shatters $ x_1,\ldots,x_{n} $ with the same witness. We can now upper and lower bound the Rademacher complexity of $ \dlh $ as follows
\begin{align}\label{eq:proofsketch10}
    c\gamma_0\eps/2 \leq \e_{\sigma\sim \{ -1,1 \}^{n} }\left[\sup_{f\in \dlh} \sum_{i=1}^{n}\sigma_{i}(f(x_{i})-r_{i})/n\right] = \e_{\sigma\sim \{ -1,1 \}^{n} }\left[\sup_{f\in \cH} \sum_{i=1}^{n}\sigma_{i}f(x_{i})\right]  \leq c' \cdot \sqrt{\frac{d}{n}},
\end{align}
for a constant $c'>0$. The first inequality holds because for any $\sigma \in \{-1,1\}^n$, by definition of $c\gamma_0\eps/2$-shattering, there is an $f \in \dlh$ with $ \sigma_{i}(f(x_{i})-r_{i})\geq c\gamma_0\eps/2$ for all $i$. The equality holds because $-\sigma_i r_i/n$ is independent of $f$, and thus can be moved outside the $\sup$, and we have $\e[\sigma_i r_i/n] = \e[\sigma_i]r_i/n=0$. Furthermore, observe that in the equality, we also replace $\sup_{f \in \dlh}$ by $\sup_{f \in \cH}$. This is true since for any convex combination $f \in \dlh$ with $f = \sum_j \alpha_j h_j$ we have $ \sum_{i}\sigma_{i}f(x_{i})=\sum_{j}\alpha_{j}\sum_{i}\sigma_{i}h_{j}(x_{i})\leq \sup_{h\in\cH}\sum_{i}\sigma_{i}h(x_{i}) $ implying $ \sup_{f\in\dlh} \sum_{i}\sigma_{i}f(x_{i})\leq \sup_{h\in\cH}\sum_{i}\sigma_{i}h(x_{i}).$ Furthermore, since $ \cH\subseteq\dlh $ the opposite inequality also holds so we have an equality. The last inequality is by classic bounds on the Rademacher complexity of classes with bounded VC-dimension, due to a bound by \cite{dudley} [see e.g. \cite{rademacherboundlecturenotes}, Theorem 7.2]. By rearrangement of \cref{eq:proofsketch10} we conclude that $ n= O(d (\gamma_0\eps)^{-2}) $ as claimed, which concludes the proof sketch.

\subsection{Majority-of-3}
\label{sec:sketchmaj3}
We finally describe the main ideas in our proof that Majority-of-3 AdaBoosts achieves an optimal in-expectation error of $O(d/(\gamma^2m))$ as stated in \cref{cor:maj3intro}. We will also explain why it is crucial for this result, that the logarithmic factors in our margin generalization bound~\cref{thm:mainintro} are $\ln(\gamma^2 m/d)$ and not $\ln(m/d), \ln(m)$ or $\ln(\gamma m/d)$. We remark that the later bounds can not be turned into $\ln(\gamma^2 m/d)$ by multiplying by some constant factor, for instance $ 2\ln{\left(\gamma m/d \right)}=\ln{\left(\gamma^{2}m^{2}/d^{2} \right)} \not\leq \ln{\left(\gamma^{2}m/d \right)}$ . Let $\cD$ be the unknown data distribution over $\cX$ and let $c^{*} \in \{-1,1\}^\cX$ be the unknown target concept.
Recall that if AdaBoost is run for $\Omega(\gamma^{-2} \ln m)$ iterations with a $\gamma$-weak learner $\cW$, then it produces a voting classifier with margins $\Omega(\gamma)$ on all examples in the training sequence $\rS \sim \cD_{c^{*}}^m$. We now use an analysis idea by~\cite{majorityofthree}, used to show that the Majority-of-3 Empirical Risk Minimizers has an optimal in-expectation error for realizable PAC learning. 

Consider partitioning a training sequence $\rS \sim \cD^{(2k-1)m}_{c^{*}}$ into $2k-1$ equal sized training sequences $\rS_1,\dots\rS_{2k-1}$ (with $k=2$ for Majority-of-3 and $k=3$ for Majority-of-5) of $m$ training examples each (rescaling $m$ by $2k-1$ recovers the guarantees for a training sequence of size $m$). If we run AdaBoost on each to obtain voting classifiers $f_{\rS_1},\dots,f_{\rS_{2k-1}}$, then each $f_{\rS_i}$ has margins $\Omega(\gamma)$ on all of $\rS_i$. For concreteness, let us say all margins are at least $\gamma/2$. Furthermore, for any point $x \in \cX$, we have that if the majority vote $\maj(f_{\rS_{1}},\cdots,f_{\rS_{2k-1}})$ errs on $x$, where we define the majority vote as $ \maj(f_{\rS_{1}},\ldots,f_{\rS_{2k-1}})=\sign(\sum_{i=1}^{2k-1}\sign(f_{i})) $, then at least $k$ of the voting classifiers err on $x$. Let us fix an $x \in \cX$ and denote by $p_x:=\p_{\rS_{i}}[f_{\rS_{i}}(x)\not=c^{*}(x)]$ the probability that $f_{\rS_i}$ errs on $x$, where the probability is over the random choice of $\rS_i$. Observe that since the training sequences $\rS_i$ are i.i.d., this is the same probability for each $\rS_i$. Moreover, by independence of the $\rS_i$'s, we have that the probability that a fixed set of $k$ of the voting classifiers all err on $x$ is precisely $p_x^k$. A union bound over all $\binom{2k-1}{k} \leq 2^{2k}$ choices of $k$ voting classifiers implies that
\[
  \Pr_{\rS}[\maj(f_{\rS_{1}},\cdots,f_{\rS_{2k-1}})(x) \neq c^{*}(x)] \leq 2^{2k} p_x^k.
\]
By swapping the order of expectation, we can bound the expected error of $\maj(f_{\rS_{1}},\cdots,f_{\rS_{2k-1}})$ as follows
\begin{align*}
  \e_{\rS}[\ls_{\cD_{c^{*}}}(\maj(f_{\rS_{1}},\cdots,f_{\rS_{2k-1}}))] &=\e_{(\rx,c^{*}(\rx)) \sim \cD_{c^{*}}}[\Pr_{\rS}[ \maj(f_{\rS_{1}},\cdots,f_{\rS_{2k-1}})\neq c^{*}(\rx)]] \\
  &=\e_{\rx \sim \cD}[2^{2k} p_\rx^k].
\end{align*}
Using the approach in~\cite{majorityofthree}, we now partition the input domain $\cX$ into regions $R_i$, such that $R_i = \{x \in \cX : p_x \in (2^{-i-1}, 2^{-i}]\}$ for $i=0,\dots,\infty$. Letting $\Pr[R_i]$ denote $\Pr_{\rx \sim \cD}[\rx \in R_i]$ and using the notation $\e_{\rx \sim \cD}[ \cdot \mid R_i]$ to denote the conditional expectation, when conditioning on $\rx \in R_i$, we now have that
\begin{align}
  \e_{\rx \sim \cD}[2^{2k} p_\rx^k] &= \sum_{i=0}^\infty \e_{\rx \sim \cD}[2^{2k} p_\rx^k \mid R_i] \Pr[R_i] \nonumber\\
  &\leq 2^{2k} \cdot \sum_{i=0}^\infty 2^{-ik} \Pr[R_i].\label{eq:boundexp}
\end{align}
The goal is thus to bound $\Pr[R_i]$. For this, the key is to exploit that $p_x > 2^{-i-1}$ for $x \in R_i$. Let $m_i = \Pr[R_i] m$ denote the expected number of samples from $R_i$ in a training sequence $\rS_j \sim \cD_{c^{*}}^m$. Since AdaBoost produces a voting classifier with margins $\Omega(\gamma)$ on all points in its training data set, it in particular has margins $\Omega(\gamma)$ on all data points in $\rS_j \cap R_i$. We note that these samples are i.i.d.\ from the conditional distribution of a $\rx \sim \cD$, conditioned on $\rx \in R_i$. Let us denote this conditional distribution by $\cD_{c^{*}} \mid R_i$. We can now invoke our new margin generalization bound in~\cref{thm:mainintro}, and conclude that
\begin{align}
  \e_{\rS_j}[\ls_{\cD_{c^{*}} \mid R_i}(f_{\rS_j})] = O\left(\frac{d \func(\gamma^2 m_i/d)}{\gamma^2 m_i} \right), \label{eq:fromgenbound}
\end{align}
with $\func(x)=\ln(x)\ln^2(\ln x)$.
Note that~\cref{thm:mainintro} actually gives a high probability guarantee, which by some calculations implies the above guarantee on the expected error. The crucial point in these calculations is that we invoke~\cref{thm:mainintro} with the conditional distribution $\cD_{c^{*}} \mid R_i$ and $\ls_{\rS_j\cap R_{i}}^{\gamma/2}(f_{\rS_j}) = 0$ since we have all margins at least $\gamma/2$, and we have $m_i$ samples from this distribution (in expectation). On the other hand, we also have by definition of $R_i$, that \cref{eq:fromgenbound} can be lower bouned as $\e_{\rS_j}[\ls_{\cD_{c^{*}} \mid R_i}(f_{\rS_j})] > 2^{-i-1}$. Writing $x = \gamma^2 m_i/d$ for short, we thus conclude that
\begin{align*}
  \frac{\func(x)}{x} = \Omega\left(2^{-i}\right) \Rightarrow x = O\left(i \ln^2(i) 2^{i} \right) \Rightarrow m_i = O\left( \frac{d i \ln^2(i) 2^i}{\gamma^2}\right) \Rightarrow \Pr[R_i] = O\left( \frac{d i \ln^2(i)2^i}{\gamma^2m}\right),
\end{align*}
where the last equality uses that $ m_{i}=\p[R_{i}]n$.
Inserting this in~\cref{eq:boundexp} bounds the expected error by (for constant $k$):
\begin{align}
\e_{\rS}[\ls_{\cD_{c^{*}}}(\maj(f_{\rS_{1}},\cdots,f_{\rS_{2k-1}}))] &= O\left( \sum_{i=0}^\infty \frac{2^{-ik} i \ln^2(i) 2^i d}{\gamma^2 m}\right).\label{eq:infsum}
\end{align}
Inserting $k=2$ (corresponding to Majority-of-3) gives the desired $O(d/(\gamma^2 m))$ as the $2^{-ik}=2^{-2i}$ decreases fast enough to cancel the $i\ln^2(i) 2^i$ factors.

\paragraph{Failure of Previous Bounds.}
Let us now discuss why the $\func(\gamma^2 m/d)$ factor is crucial compared to $\ln(m)$, $\ln(\gamma m/d)$ and $\ln(m/d)$ factors in the above analysis. Consider again~\cref{eq:fromgenbound} and assume for simplicity that the generalization bound had instead given us
\begin{align}\label{eq:proofsketch0}
  \e_{\rS_i}[\ls_{\cD_{c^{*}} \mid R_i}(f_{\rS_i})] = O\left(\frac{d \ln(\gamma m_i/d)}{\gamma^2 m_i} \right),
\end{align}
i.e.\ a slightly suboptimal dependency on $\gamma$ inside the $\ln(\cdot) \ln^2(\ln(\cdot ))$. We would then get the inequality
\[
  \frac{d \ln(\gamma m_i/d)}{\gamma^2 m_i} = \Omega(2^{-i}).
\]
Now again letting $ x=\gamma^{2}m_{i}/d $ then the above can be shown to imply
\begin{align*}
  \frac{\ln(x/\gamma)}{x} = \Omega\left(2^{-i}\right) \Rightarrow x = O\left(\ln{\left(2^i/\gamma \right)} 2^{i} \right) \Rightarrow \Pr[R_i] = O\left( \frac{d\ln{\left(2^i/\gamma \right)}2^i}{\gamma^2m}\right).
\end{align*}
Now plugging this bound $ \Pr[R_{i}] $  into \cref{eq:boundexp} (with $ k=2 $), we get the following error bounds for $ \e_{\rS}[\ls_{\cD_{c^{*}}}(\maj(f_{\rS_{1}},f_{\rS_{2}},f_{\rS_{3}}))]  $ of
\[
  O\left( \sum_{i=0}^\infty \frac{2^{-2i}d \ln(2^i/\gamma) 2^i }{\gamma^2 m}\right) = O\left(\frac{d \ln(1/\gamma)}{\gamma^2 m} \right).
\]
The obtained error bound of $ \e_{\rS}[\ls_{\cD_{c^{*}}}(\maj(f_{\rS_{1}},f_{\rS_{2}},f_{\rS_{3}}))]  $ thus has a superfluous $ \ln{\left(1/\gamma \right)}$ factor -  if the bound in \cref{eq:proofsketch0} where with $ \ln{\left(\gamma m/d \right)} $ and $ \ln{\left(m/d \right)} $ instead the above analysis would again give a superfluous factor. 

These shortcomings of previous margins bounds used in the above analysis of the expected error $ \e_{\rS}[\ls_{\cD_{c^{*}}}(\maj(f_{\rS_{1}},f_{\rS_{2}},f_{\rS_{3}}))]  $, introducing superfluous factors, is precisely the reason why previous work needed a Majority-of-5. Being unable to use the margin generalization bounds with sub-optimal $\ln( \cdot )$ factors, the work~\cite{manyfacesofoptimalweaktostronglearning} instead relied on the much weaker guarantee that any voting classifier $f$ with margins $\gamma$ has  $\ls_{\cD_{c^{*}}}(f) = O(\sqrt{d/(\gamma^2 m)}),$ where this bound can be obtained by following the steps of~\cite{boostingbook} [page 107-111] and using the stronger bound on the Rademacher complexity for a function class with VC-dimension $ d $ of $ O(\sqrt{d/m}) $ due to \cite{dudley} [See e.g. Theorem 7.2 \cite{rademacherboundlecturenotes}], instead of the weaker $ O(\sqrt{d\ln{\left(m/d \right)}/m})$ used in \cite{boostingbook}. This bound has the right behaviour for $m \approx d/\gamma^2$ unlike the bounds with sub-optimal logarithmic factors and results in the guarantee $\Pr_{\rx\sim\cD}[R_i] = O(d2^{2i}/(\gamma^{2}m))$ instead of $O(di \ln^2(i) 2^i/(\gamma^{2}m))$. This needs $k=3$ for $2^{-ik}$ to dominate $2^{2i}$ in $\sum_{i=0}^\infty 2^{-ik} 2^{2i}$ from ~\cref{eq:infsum}, whereas it suffices for us with $k=2$ since we only need to bound $\sum_{i=0}^\infty 2^{-ik} i \ln^2(i) 2^{i}$, leading to the more natural Majority-of-3 instead of Majority-of-5.

\paragraph{Organization of Paper.}
The following sections are organized as follows. In~\cref{sec:upperbound} we prove the margin generalization bound in \cref{eq:proofsketch2}. In~\cref{sec:cover} we prove the bound on the covering number of the clipped function class $\dlh_{\left\lceil\gamma \right\rceil}$.  In~\cref{sec:finalbound}, we put together the results from~\cref{sec:upperbound} and~\cref{sec:cover} to prove the main margin generalization bound in~\cref{thm:mainintro}. In~\cref{sec:mainmajoritythree} we prove the main result on Majority-of-3 in \cref{cor:maj3intro}.  

\section{Upper bound}\label{sec:upperbound}
In this section we prove
\cref{lem:marginbound}, which is used in a
union bound over suitable $ \gamma_{0},\gamma_{1} $ and $ \tau_{0} $
and $ \tau_{1},$ giving the margin bound in~\cref{thm:mainintro} for all $ f\in\dlh $ and margins $ 0<\gamma\leq1.$ We now state \cite{lem:marginbound}.
 \begin{lemma}\label{lem:marginbound}
    For a hypothesis class $ \cH \subseteq \{  -1,1\}^{\cX} $, margin thresholds  $ 0<\gamma_{0}\leq \gamma_{1}\leq 1 $, error thresholds $ 0\leq\tau_{0}\leq \tau_{1} $, failure parameter $0<\delta<1 $, we have that
    \begin{align*}
        &\p_{\rS\sim\cD^{m}}\negmedspace\left[
        \exists \gamma \negmedspace\in\negmedspace \left[\gamma_{0},\gamma_{1}\right]
        ,\exists f\negmedspace\in\negmedspace \dlh\negmedspace: 
        \ls_{\rS}^{\gamma}(f)\negmedspace\in \negmedspace[\tau_{0},\tau_{1}] 
        ,\ls_{\cD}(f) \geq \tau_{1} 
        + \cc\left(\negmedspace\sqrt{\frac{\tau_{1}\cdot 2 \ln{\left(\frac{e}{\delta} \right)}}{m}} 
        +\frac{2\ln{\left(\frac{e}{\delta} \right)}}{m}\negmedspace\right)
        \negmedspace\right] 
        \\
        &\leq\delta \cdot \sup_{X\in \cX^{2m}}   \cN_{\infty}(X,\dlh_{\left\lceil2\gamma_{1}\right\rceil},\frac{\gamma_{0}}{2}) .
    \end{align*}

 \end{lemma}

 \begin{proof}
    We first notice that if $ \tau_{1}>1 $ then we are done by $
    \ls_{\cD}(f)\leq 1,$ thus we assume for the remaining of the proof
    that $ \tau_{1}\leq 1.$ For ease of notation, let $
    \beta=\left(\sqrt{\tau_{1}\cdot 2 \ln{\left(\frac{e}{\delta}
          \right)}/m} +2\ln{\left(\frac{e}{\delta} \right)}/m\right)$ in the following.   
    We start by showing that
    \begin{align}\label{eq:marginbound-1}
        &\p_{\rS\sim\cD^{m} }\left[
        \exists \gamma \in \left[\gamma_{0},\gamma_{1}\right],
         \exists f \in \dlh: 
         \ls_{\rS}^{\gamma}(f)\in [\tau_{0},\tau_{1}] ,\ls_{\cD}(f) \geq \tau+64\beta\right] 
        \\
        &\leq 2
        \p_{\rS,\rS'\sim\cD^{m}}\left[\exists \gamma \in \left[\gamma_{0},\gamma_{1}\right],\exists f\in \dlh: \ls_{\rS}^{\gamma}(f)\in [\tau_{0},\tau_{1}] ,\ls_{\rS'}(f) \geq \tau_{1} + 32\beta\right]\nonumber.
    \end{align}
    To this end, we first notice that if $ f\in \dlh $ is such that $ \ls_{\cD}(f)> \frac{2\ln{\left(\frac{e}{\delta} \right)}}{m} $, then 
   by Chernoff, we have that 
   \begin{align*}
    \p_{\rS'}\left[\ls_{\rS'}(f)\leq \left(1-\sqrt{\frac{2\ln{\left(\frac{e}{\delta} \right)}}{m\ls_{\cD}(f)}}\right)\ls_{\cD}(f)\right]\nonumber
    \leq
     \exp{\left(-\frac{2\ln{\left(\frac{e}{\delta} \right)}}{2} \right)}\nonumber
     \leq  \frac{\delta}{e}.\nonumber
   \end{align*}
   This implies that with probability at least $ 1-\frac{\delta}{e}$ over $ \rS' $ , we have 
   \begin{align}
    \ls_{\rS'}(f)\geq \ls_{\cD}(f)-\sqrt{\frac{\ls_{\cD}(f)\cdot 2\ln{\left(\frac{e}{\delta} \right)}}{m}}.\nonumber
   \end{align} 
   Now,  for $ a>0,$  $ x-\sqrt{ax} $ is increasing for $ x\geq a/4,$ since it has derivative $ 1-\frac{a}{2\sqrt{ax}}$. Thus, for  $\ls_{\cD}(f) \geq\tau_{1}+64\beta =\tau_{1} +\cc\left(\sqrt{\frac{\tau_{1}\cdot 2 \ln{\left(\frac{e}{\delta} \right)}}{m}} +\frac{2\ln{\left(\frac{e}{\delta} \right)}}{m}\right) \geq \frac{2\ln{\left(\frac{e}{\delta} \right)}}{m}$, we have by the above with $ a=\frac{2\ln{\left(\frac{e}{\delta} \right)}}{m},$ that with probability at least $ 1-\frac{\delta}{e}$ over $ \rS',$  we have
   
   \begin{align}\label{eq:marginbound-2}
    &\ls_{\rS'}(f)\geq \ls_{\cD}(f)-\sqrt{\frac{\ls_{\cD}(f)\cdot 2\ln{\left(\frac{e}{\delta} \right)}}{m}}
    \\
    &\geq \tau_{1} 
    +\cc\left(\sqrt{\frac{\tau_{1}\cdot 2 \ln{\left(\frac{e}{\delta} \right)}}{m}} +\frac{2\ln{\left(\frac{e}{\delta} \right)}}{m}\right) 
    -\sqrt{
        \frac{
        \left(
        \tau_{1} 
        +\cc\left(\sqrt{\frac{\tau_{1}\cdot 2 \ln{\left(\frac{e}{\delta} \right)}}{m}} 
        +\frac{2\ln{\left(\frac{e}{\delta} \right)}}{m}\right)\right)\cdot 2\ln{\left(\frac{e}{\delta} \right)}}{m}},
        \nonumber
   \end{align} 
where the square root term can be upper bounded using $ a+\sqrt{ab}+b\leq2(a+b)  $ and $ \sqrt{a+b} \leq \sqrt{a}+\sqrt{b}$ for $ a,b>0,$ as follows
   \begin{align}
    &\sqrt{
        \frac{
        \left(
        \tau_{1} 
        +\cc\left(\sqrt{\frac{\tau_{1}\cdot 2 \ln{\left(\frac{e}{\delta} \right)}}{m}} 
        +\frac{2\ln{\left(\frac{e}{\delta} \right)}}{m}\right)\negmedspace\right)\cdot 2\ln{\left(\frac{e}{\delta} \right)}}{m}}
        \leq 
    \sqrt{
        \frac{\cc
        \left(
        \tau_{1} 
        +\sqrt{\frac{\tau_{1}\cdot 2 \ln{\left(\frac{e}{\delta} \right)}}{m}} 
        +\frac{2\ln{\left(\frac{e}{\delta} \right)}}{m}\negmedspace\right)\cdot 2\ln{\left(\frac{e}{\delta} \right)}}{m}}\nonumber
    \\
    &\leq\sqrt{
        \frac{2\cdot \cc
        \left(
        \tau_{1} 
        +
        \frac{2\ln{\left(\frac{e}{\delta} \right)}}{m}\right)\cdot 2\ln{\left(\frac{e}{\delta} \right)}
        }{m}}\nonumber
        \leq\sqrt{2\cdot64}\left(
          \sqrt{ 
            \frac{
            \tau_{1}\cdot2 
            \ln{\left(\frac{e}{\delta} \right)}
            }{m}}
            +
            \frac{2\ln{\left(\frac{e}{\delta}\right)}
            }{m}\right)\nonumber.
   \end{align}
   Now using the above upper bound combined with \cref{eq:marginbound-2} we conclude that
   \begin{align}\label{eq:marginbound1}
    &\ls_{\rS'}(f)
        \geq
        \tau_{1} 
        +\left(\cc-\sqrt{2\cdot\cc}\right)\sqrt{\frac{\tau_{1}\cdot 2 \ln{\left(\frac{e}{\delta} \right)}}{m}} +\left(\cc-\sqrt{2\cdot\cc}\right)\frac{2\ln{\left(\frac{e}{\delta} \right)}}{m}
        \\
        &\geq \tau_{1}
        +
        32
        \left(\sqrt{\frac{\tau_{1}\cdot 2 \ln{\left(\frac{e}{\delta} \right)}}{m}} +\frac{2\ln{\left(\frac{e}{\delta} \right)}}{m}\right)=\tau_{1}+32\beta \nonumber,
   \end{align} 
where the last equality uses that $ \beta=\left(\sqrt{\tau_{1}\cdot 2 \ln{\left(\frac{e}{\delta} \right)}/m} +2\ln{\left(\frac{e}{\delta} \right)}/m\right)$.
   Thus, we conclude by the above and the law of total probability that 
\begin{align}\label{eq:marginbound16}
    &\p_{\rS,\rS'\sim\cD^{m}}\left[\exists \gamma \in \left[\gamma_{0},\gamma_{1}\right],\exists f\in \dlh: \ls_{\rS}^{\gamma}(f)\in [\tau_{0},\tau_{1}] ,\ls_{\rS'}(f) \geq \tau_{1} + 32\beta\right]
    \\
    &=
    \p_{\rS,\rS'\sim\cD^{m}} \Bigg[\exists \gamma \in \left[\gamma_{0},\gamma_{1}\right], \exists f\in \dlh: \ls_{\rS}^{\gamma}(f)\in [\tau_{0},\tau_{1}] ,\ls_{\rS'}(f) \geq \tau_{1} 
     +32\beta\nonumber
    \\
    &\mid \exists \gamma \in \left[\gamma_{0},\gamma_{1}\right],\exists f \in \dlh: \ls_{\rS}^{\gamma}(f)\in [\tau_{0},\tau_{1}] ,\ls_{\cD}(f) \geq \tau_{1}  +\cc\beta\Bigg]\nonumber 
    \\
    &\cdot
    \p_{\rS}\left[
        \exists \gamma \in \left[\gamma_{0},\gamma_{1}\right],
        \exists f \in \dlh:
     \ls_{\rS}^{\gamma}(f)\in [\tau_{0},\tau_{1}] ,\ls_{\cD}(f) \geq \tau_{1} +\cc\beta\right] \nonumber
    \\
    &\geq
     \left(1-\frac{\delta}{e}\right) \p_{\rS}\left[
    \exists \gamma \in \left[\gamma_{0},\gamma_{1}\right],
     \exists f \in \dlh: 
     \ls_{\rS}^{\gamma}(f)\in [\tau_{0},\tau_{1}] ,\ls_{\cD}(f) \geq \tau_{1} +\cc\beta\right] 
     \nonumber
   \end{align}
where the second to last inequality follows by the law of total probability and the last inequality by \cref{eq:marginbound1}, 
whereby we have shown \cref{eq:marginbound-1}.

We now bound the probability of the expression on the second line of
\cref{eq:marginbound-1}  by
\[
   \sup_{X\in \cX^{2m}}
   \cN_{\infty}(X,\dlh_{\left\lceil2\gamma_{1}\right\rceil},\frac{\gamma_{0}}{6})\frac{\delta}{e},
 \]
 which would conclude the proof.

Now consider any $ \gamma $ such that $\gamma_{0}\leq \gamma\leq \gamma_{1}.$ We now recall that for a function $ f $, we defined $ f_{\left\lceil2\gamma_{1}\right\rceil} $ as follows:
\begin{align*}
    f_{\left\lceil2\gamma_{1}\right\rceil}(x)=\begin{cases}
        2\gamma_{1} &\text{ if }f(x)> 2\gamma_{1}
        \\
        -2\gamma_{1} &\text{ if } f(x)<- 2\gamma_{1}
        \\
        f(x) &\text{ else }.
    \end{cases}
\end{align*}
Thus,  $ f_{\left\lceil2\gamma_{1}\right\rceil}(x) $ has the same sign as $ f(x) $  and is equal to $ f(x), $ when $ |f(x)|\leq 2\gamma_{1}.$  Using the above observations we conclude by $ y\in\{  -1,1\}  $ and $ \gamma\leq \gamma_{1} $   that $ \ind\{   f(x)y\leq \gamma \}=\ind\{   f_{\left\lceil2\gamma_{1}\right\rceil}(x)y\leq \gamma\}$  and   $ \ind\{   f(x)y\leq 0 \}=\ind\{   f_{\left\lceil2\gamma_{1}\right\rceil}(x)y\leq 0 \}$. Thus, we get that
\begin{align}\label{eq:marginbound19}
    &\p_{\rS,\rS'\sim\cD^{m}}\left[ \exists \gamma \in \left[\gamma_{0},\gamma_{1}\right],\exists f\in \dlh :  \ls_{\rS}^{\gamma}\left(f\right)
    \leq \tau_{1}, \ls_{\rS'}(f) \geq 
    \tau_{1} 
        + 32\beta\right]
        \\
        &=\p_{\rS,\rS'\sim\cD^{m}}\left[ \exists \gamma \in \left[\gamma_{0},\gamma_{1}\right],\exists f\in \dlh_{\left\lceil2\gamma_{1}\right\rceil} :  \ls_{\rS}^{\gamma}\left(f\right)
        \leq \tau_{1}, \ls_{\rS'}(f) \geq 
        \tau_{1} 
            + 32\beta\right]. \nonumber
\end{align}

By $ \rS $ and $ \rS' $ being i.i.d., we may see $ \rS $ and $ \rS' $ as being drawn in the following way: First, we draw $ \bar{\rS}\sim \cD^{2m} $ and then let $ \rS $ be created by drawing $ m $ times without replacement from $ \bar{\rS} $ and $ \rS'=\bar{\rS}\backslash\rS $.  We denote this way of drawing $ \rS $ and $ \rS' $ from $ \bar{\rS} $ as  $ (\rS,\rS')\sim \bar{\rS} $. We then have that,
   \begin{align}\label{eq:marginbound14}
    &\p_{\rS,\rS'\sim\cD^{m}}\left[\exists \gamma \in \left[\gamma_{0},\gamma_{1}\right],\exists f\in \dlh_{\left\lceil2\gamma_{1}\right\rceil} :  \ls_{\rS}^{\gamma}\left(f\right)
    \leq \tau_{1}, \ls_{\rS'}(f) \geq 
    \tau_{1} 
        + 32\beta\right]
    \\
    &=
    \p_{\bar{\rS}\sim \cD^{2m}}\Bigg[\p_{(\rS,\rS')\sim\bar{\rS}}\Bigg[\exists \gamma \in \left[\gamma_{0},\gamma_{1}\right],\exists f\in \dlh_{\left\lceil2\gamma_{1}\right\rceil} :  \ls_{\rS}^{\gamma}\left(f\right)
    \leq \tau_{1}, \ls_{\rS'}(f) \geq 
    \tau_{1} 
        + 32\beta\Bigg]\Bigg]\nonumber
    \\
    &\leq
    \sup_{Z\in (\cX\times\{  -1,1\}) ^{2m}}
    \p_{(\rS,\rS')\sim Z}\bigg[\exists \gamma \in \left[\gamma_{0},\gamma_{1}\right],\exists f\in \dlh_{\left\lceil2\gamma_{1}\right\rceil} :  \ls_{\rS}^{\gamma}\left(f\right)
    \leq \tau_{1}, \ls_{\rS'}(f) \geq 
    \tau_{1} 
        + 32\beta\bigg].\nonumber
   \end{align}
   Let now $ Z=(X,Y)\in (\cX\times\{  -1,1\}) ^{2m} $ and $ \Net $ be a $ \frac{\gamma_{0}}{2} $-cover for $ \dlh_{\left\lceil2\gamma_{1}\right\rceil} $ on $ X $   with respect to $ \infty $-norm of minimal size. That is, for any $ f\in \dlh_{\left\lceil2\gamma_{1}\right\rceil} $, we have
   \begin{align*}
    \min_{f'\in \Net}\max_{x\in X}  \left| f(x)-f'(x) \right| \leq \frac{\gamma_{0}}{2}.
   \end{align*}  
   and any other cover satisfying the above has size at least $|\Net|$. 
   We notice that this implies that if $ \gamma $ is such that $
   \gamma_{0}\leq\gamma\leq\gamma_{1} $ and $ f\in
   \dlh_{\left\lceil2\gamma_{1}\right\rceil} $, then for the function
   $f'$ closest to $ f $ in $N$, i.e.\ $f'=\text{argmin}_{f'\in\Net
   }\max_{x\in X}  \left| f(x)-f'(x) \right|,$ with ties broken
   arbitrarily, it holds for any $ (x,y)\in Z$ with $ f(x)y\leq0 $ that$ f'(x)y=f(x)y+(f'(x)-f(x))y\leq0+\frac{\gamma_{0}}{2} \leq \frac{\gamma_{0}}{2}$. 
   Similarly for any $ (x,y)\in Z $ such that $ f(x)y> \gamma $, we have $ f'(x)y=f(x)y+(f'(x)-f(x))y>\gamma-\frac{\gamma_{0}}{2} >\frac{\gamma_{0}}{2}$, where the last inequality follows from $ \gamma\geq \gamma_{0}$. 
   Thus, we conclude that for any $ \gamma $ such that $
   \gamma_{0}\leq \gamma\leq\gamma_{1} $ and $
   f\in\dlh_{\left\lceil2\gamma_{1}\right\rceil} $, there exists $
   f'\in \Net$ such that   $\ls_{\rS'}^{\frac{\gamma_{0}}{2}}(f')\geq \ls_{\rS'}(f)$ and that $ \ls_{\rS}^{\gamma}\left(f\right)\geq \ls_{\rS}^{\frac{\gamma_{0}}{2}}\left(f'\right) $.  
   By the above and the union bound, we conclude that
   \begin{align}\nonumber
    &\p_{\rS,\rS'\sim Z}\bigg[\exists \gamma \in \left[\gamma_{0},\gamma_{1}\right],\exists f\in \dlh_{\left\lceil2\gamma_{1}\right\rceil} :  \ls_{\rS}^{\gamma}\left(f\right)
    \leq \tau_{1}, \ls_{\rS'}(f) \geq 
    \tau_{1} 
        + 32\beta\bigg]
    \\
    &\leq
    \p_{\rS,\rS'\sim Z}\bigg[\exists f\in N :  \ls_{\rS}^{\frac{\gamma_{0}}{2}}\left(f\right)
    \leq \tau_{1}, \ls_{\rS'}^{\frac{\gamma_{0}}{2}}(f) \geq 
    \tau_{1} 
        + 32\beta\bigg]\nonumber
    \\
    &\leq
    \sum_{f\in \Net}
    \p_{\rS,\rS'\sim Z}\bigg[  \ls_{\rS}^{\frac{\gamma_{0}}{2}}\left(f\right)
    \leq \tau_{1}, \ls_{\rS'}^{\frac{\gamma_{0}}{2}}(f) \geq 
    \tau_{1} 
        + 32\beta\bigg]. \label{eq:marginbound13}
   \end{align}
   We now show that for each $ f\in \Net $, each term/probability in the sum is at most $ \frac{\delta}{e}, $ so the sum is at most $ \cN_{\infty}(X,\dlh_{\left\lceil2\gamma_{1}\right\rceil},\frac{\gamma_{0}}{2}) \frac{\delta}{e} $. Which implies by \cref{eq:marginbound14}, that \cref{eq:marginbound19} is bounded by $\sup_{X\in \cX^{2m}}   \cN_{\infty}(X,\dlh_{\left\lceil2\gamma_{1}\right\rceil},\frac{\gamma_{0}}{2}) \frac{\delta}{e} $ as claimed below \cref{eq:marginbound16} and would conclude the proof. To the end of showing the above let  $ f\in \Net $ for now.

   Recall that $ \beta=\left(\sqrt{\frac{\tau_{1}\cdot 2 \ln{\left(\frac{e}{\delta} \right)}}{m}} 
+\frac{2\ln{\left(\frac{e}{\delta} \right)}}{m}\right).$ Let $ 2\mu $  denote 2 times the fraction of points in $ Z $ where $ f $ has less than  $ \gamma_{0}/2 $-margin, i.e.,    
$$ 2\mu= \ls_{\rS}^{\frac{\gamma_{0}}{2}}(f)+\ls_{\rS'}^{\frac{\gamma_{0}}{2}}(f)= \left| \{ (x,y)\in Z: f(x)y\leq \frac{\gamma_0}{2} \} \right| /m. $$
We notice that $ \mu $ is the expectation of  $  \ls_{\rS}^{\frac{\gamma_{0}}{2}}(f)$ and that for the probability in \cref{eq:marginbound13} to be nonzero, it must be the case that $ \mu \geq 32\frac{\ln{\left(\frac{e}{\delta} \right)}}{m} $. Thus, we assume this is the case.
We notice that if $ \mu > \frac{2\ln{\left(\frac{e}{\delta} \right)}}{m}$, we have by Chernoff (which due to \cite{withreplacementchernoff}[see section 6] also holds when sampling without replacement) that  
\begin{align*}
    \p_{\rS,\rS'\sim Z}\left[\ls_{\rS}^{\frac{\gamma_{0}}{2}}(f)\leq \left(1-\sqrt{\frac{2\ln{\left(\frac{e}{\delta} \right)}}{\mu m}}\right)\mu \right]
    \leq  \exp{\left(-\frac{2\ln{\left(\frac{e}{\delta} \right)}\mu m}{2\mu m} \right)}\leq  \frac{\delta}{e}
\end{align*}
Thus, with probability at least $ 1-\frac{\delta}{e} $ over $ \rS,\rS' $, we have by definition of $ \mu $ that
\begin{align}\label{eq:marginbound7}
    &\ls_{\rS}^{\frac{\gamma_{0}}{2}}(f)
    \geq 
    \mu-\sqrt{\frac{\mu 2 \ln{\left(\frac{e}{\delta} \right)}}{m}}
    =
    \frac{\ls_{\rS}^{\frac{\gamma_{0}}{2}}(f)+\ls_{\rS'}^{\frac{\gamma_{0}}{2}}(f)}{2}-
\sqrt{\frac{\left(
    \ls_{\rS}^{\frac{\gamma_{0}}{2}}(f)+\ls_{\rS'}^{\frac{\gamma_{0}}{2}}(f)
    \right)\ln{\left(\frac{e}{\delta} \right)}}{m}}
    \\
    &\geq
    \frac{\ls_{\rS}^{\frac{\gamma_{0}}{2}}(f)+\ls_{\rS'}^{\frac{\gamma_{0}}{2}}(f)}{2}-
\sqrt{\frac{
    \ls_{\rS}^{\frac{\gamma_{0}}{2}}(f)
    \ln{\left(\frac{e}{\delta} \right)}}{m}}
    -
\sqrt{\frac{
    \ls_{\rS'}^{\frac{\gamma_{0}}{2}}(f)
    \ln{\left(\frac{e}{\delta} \right)}}{m}}
    \tag{by $ \sqrt{a+b}\leq \sqrt{a}+\sqrt{b} $ }
    \\
    &\Rightarrow
    \frac{\ls_{\rS}^{\frac{\gamma_{0}}{2}}(f)}{2}+
\sqrt{\frac{
    \ls_{\rS}^{\frac{\gamma_{0}}{2}}(f)
    \ln{\left(\frac{e}{\delta} \right)}}{m}}
\geq
\frac{\ls_{\rS'}^{\frac{\gamma_{0}}{2}}(f)}{2}-
\sqrt{\frac{
    \ls_{\rS'}^{\frac{\gamma_{0}}{2}}(f)
    \ln{\left(\frac{e}{\delta} \right)}}{m}}.
    \tag{by rearrangement}
\end{align}
Recall that $ \beta=\left(\sqrt{\frac{\tau_{1}\cdot 2 \ln{\left(\frac{e}{\delta} \right)}}{m}} 
+\frac{2\ln{\left(\frac{e}{\delta} \right)}}{m}\right).$ We now show
that if $
\ls_{\rS'}^{\frac{\gamma_{0}}{2}}(f)\geq \tau_{1}+32\beta=\tau_{1} 
+ 32\Big(\sqrt{\frac{\tau_{1}\cdot 2 \ln{\left(\frac{e}{\delta} \right)}}{m}} 
+\frac{2\ln{\left(\frac{e}{\delta} \right)}}{m}\Big) $ and $ \ls_{\rS}^{\frac{\gamma_{0}}{2}}(f)
\leq  \tau_{1}$ then we have
\begin{align}\label{eq:marginbound17}
  \frac{\ls_{\rS}^{\frac{\gamma_{0}}{2}}(f)}{2}+
        \sqrt{\frac{
            \ls_{\rS}^{\frac{\gamma_{0}}{2}}(f)
            \ln{\left(\frac{e}{\delta} \right)}}{m}} < 
    \frac{\ls_{\rS'}^{\frac{\gamma_{0}}{2}}(f)}{2}-
    \sqrt{\frac{
       \ls_{\rS'}^{\frac{\gamma_{0}}{2}}(f)
        \ln{\left(\frac{e}{\delta} \right)}}{m}}.
        \end{align}
This implies that the event $
\ls_{\rS'}^{\frac{\gamma_{0}}{2}}(f)\geq\tau_{1}+32\beta$ and $ \ls_{\rS}^{\frac{\gamma_{0}}{2}}(f)
\leq  \tau_{1}$ is in the complement of the event $ \ls_{\rS}^{\frac{\gamma_{0}}{2}}(f)
\geq 
\mu-\sqrt{\frac{\mu 2 \ln{\left(\frac{e}{\delta} \right)}}{m}}$, which
we just argued happens with probability at least $ 1-\frac{\delta}{e}
$ over $ \rS ,\rS'$. Thus we now show that $
\ls_{\rS'}^{\frac{\gamma_{0}}{2}}(f)\geq \tau_{1} 
+ 32\beta $ and $ \ls_{\rS}^{\frac{\gamma_{0}}{2}}(f)
\leq  \tau_{1}$  implies \cref{eq:marginbound17}, which would
establish that each term in \cref{eq:marginbound13} is no more than $
\frac{\delta}{e} $ and thereby conclude the proof.

If $ \ls_{\rS}^{\frac{\gamma_{0}}{2}}(f)
\leq  \tau_{1}$, we have that 

\begin{align}\label{eq:marginbound5}
    &\frac{\ls_{\rS}^{\frac{\gamma_{0}}{2}}(f)}{2}+
    \sqrt{\frac{
        \ls_{\rS}^{\frac{\gamma_{0}}{2}}(f)
        \ln{\left(\frac{e}{\delta} \right)}}{m}}
        \leq 
        \frac{\tau_1}{2}
        + 
        \sqrt{\frac{\tau_{1}\ln{\left(\frac{e}{\delta} \right)}}{m}}
\end{align}
Now if
$
\ls_{\rS'}^{\frac{\gamma_{0}}{2}}(f)\geq \tau_1 + 32 \beta =  \tau_{1} 
+ 32\Big(\sqrt{\frac{\tau_{1}\cdot 2 \ln{\left(\frac{e}{\delta} \right)}}{m}} 
+\frac{2\ln{\left(\frac{e}{\delta} \right)}}{m}\Big) $  
and using that we argued earlier that $ x-\sqrt{ax} $ is increasing for $ x\geq \frac{a}{4} $, we get that 
\begin{align}\label{eq:marginbound6}
&\frac{\ls_{\rS'}^{\frac{\gamma_{0}}{2}}(f)}{2}-
\sqrt{\frac{
    \ls_{\rS'}^{\frac{\gamma_{0}}{2}}(f)
    \ln{\left(\frac{e}{\delta} \right)}}{m}} 
\\
&\geq
\frac{\tau_1}{2}
+ 16\Big(\sqrt{\frac{\tau_{1}\cdot 2 \ln{\left(\frac{e}{\delta} \right)}}{m}} 
+\frac{2\ln{\left(\frac{e}{\delta} \right)}}{m}\Big) 
-
\sqrt{\frac{\left(\tau_{1}
+ 32\Big(\sqrt{\frac{\tau_{1}\cdot 2 \ln{\left(\frac{e}{\delta} \right)}}{m}} 
+\frac{2\ln{\left(\frac{e}{\delta} \right)}}{m}\Big) 
    \right)\ln{\left(\frac{e}{\delta} \right)}}{m}}.\nonumber
\end{align}
Now upper bounding the second term  by using that $  a+\sqrt{ab}+b\leq 2(a+b)$ and $ \sqrt{a+b}\leq \sqrt{a}+\sqrt{b},$ 
\begin{align*}
&
\sqrt{\frac{\left(
\tau_{1}
+ 32\Big(\sqrt{\frac{\tau_{1}\cdot 2 \ln{\left(\frac{e}{\delta} \right)}}{m}} 
+\frac{2\ln{\left(\frac{e}{\delta} \right)}}{m}\Big) 
    \right)\ln{\left(\frac{e}{\delta} \right)}}{m}}
\leq
\sqrt{\frac{32\left(
\tau_{1}
+ \sqrt{\frac{\tau_{1}\cdot 2 \ln{\left(\frac{e}{\delta} \right)}}{m}} 
+\frac{2\ln{\left(\frac{e}{\delta} \right)}}{m} 
    \right)\ln{\left(\frac{e}{\delta} \right)}}{m}}
\\
&\leq
\sqrt{\frac{64\left(
\tau_{1}
+\frac{2\ln{\left(\frac{e}{\delta} \right)}}{m} 
    \right)\ln{\left(\frac{e}{\delta} \right)}}{m}}
\leq\sqrt{64}
\sqrt{\frac{
\tau_{1}\ln{\left(\frac{e}{\delta} \right)}}{m}}
+\sqrt{128}
\frac{ 
    \ln{\left(\frac{e}{\delta} \right)}}{m},
\end{align*}
we conclude from \cref{eq:marginbound6} that 
\begin{align}\label{eq:marginbound8}
    \frac{\ls_{\rS'}^{\frac{\gamma_{0}}{2}}(f)}{2}-
    \sqrt{\frac{
        \ls_{\rS'}^{\frac{\gamma_{0}}{2}}(f)
        \ln{\left(\frac{e}{\delta} \right)}}{m}}  \geq \frac{\tau_{1}}{2}+(16\sqrt{2}-\sqrt{64})\sqrt{\frac{\tau_{1}\cdot \ln{\left(\frac{e}{\delta} \right)}}{m}} 
        +(32-\sqrt{128})\frac{\ln{\left(\frac{e}{\delta} \right)}}{m}.
\end{align}
Thus, combining \cref{eq:marginbound5} and \cref{eq:marginbound8} and using that $ 16\sqrt{2}-\sqrt{64} \geq 14$ and $ 32-\sqrt{128}\geq 20,$ we conclude that if $
\ls_{\rS'}^{\frac{\gamma_{0}}{2}}(f)\geq\tau_{1}+32\beta= \tau_{1} 
+ 32\Big(\sqrt{\frac{\tau_{1}\cdot 2 \ln{\left(\frac{e}{\delta} \right)}}{m}} 
+\frac{2\ln{\left(\frac{e}{\delta} \right)}}{m}\Big) $ and $ \ls_{\rS}^{\frac{\gamma_{0}}{2}}(f)
\leq  \tau_{1}$, then it implies that 
\begin{align*}
    \frac{\ls_{\rS'}^{\frac{\gamma_{0}}{2}}(f)}{2}-
    \sqrt{\frac{
       \ls_{\rS'}^{\frac{\gamma_{0}}{2}}
        (f)\ln{\left(\frac{e}{\delta} \right)}}{m}} 
        >
        \frac{\ls_{\rS}^{\frac{\gamma_{0}}{2}}(f)}{2}+
        \sqrt{\frac{
            \ls_{\rS}^{\frac{\gamma_{0}}{2}}(f)
            \ln{\left(\frac{e}{\delta} \right)}}{m}},
        \end{align*}
which as argued earlier concludes the proof.
\end{proof}
\section{Bound on infinity cover}\label{sec:cover}

In this section, we prove the necessary result on the infinity cover of $ \Delta(\cH)_{\left\lfloor\gamma\right\rfloor} $. To this end we recall that we for a function class $ \cF $ and a point set $ X=\{ x_{1},\ldots ,x_{m} \}  \subset\cX$ of size $ m $ and precision $ \gamma $ define $ \cN_{\infty}(X,\cF,\gamma) $ to be the minimal size of a $ \gamma $-cover in infinity norm of $ \cF $ on $ X,$ i.e. the smallest size of a set of functions $ \Net $ satisfying $ \min_{f'\in \Net}\max_{x\in X} |f(x)-f'(x)|\leq \gamma .$ Furthermore, we will use $ \Ln(x)=\ln(\max(x,e)) $ for the truncated natural logarithm. We now state \cref{lem:coverfinal} and \cref{cor:coverfinal}, and show how \cref{lem:coverfinal} implies \cref{cor:coverfinal}.

\begin{lemma}\label{lem:coverfinal}
    There exist universal constants $ c'\geq c>0 $ and $ c'\geq 1 $  such that: For a point set $ X=\{ x_{1},\ldots,x_{m} \} \subseteq \cX $ of size $ m $, $ \cH \subseteq \{-1,1\}^{\cX}$ a hypothesis class of VC-dimension $ d $, $ 0<\eps<1 $,     $ 0<\alpha<\frac{1}{2} $  and $ 0<\gamma\leq 1 $  such that $ \gamma^{2}\geq  \frac{cd}{\alpha^{2}\eps^{2}m}$, we have 
   \begin{align}
       \ln(\cN_{\infty}(X,\dlh_{\left\lceil\gamma\right\rceil} ,\alpha \gamma)) \leq
       \frac{c'd}{\alpha^{2}\eps^{2}\gamma^{2}} \Ln^{1+\eps}{\left(  \frac{8\alpha m\eps^{2} \gamma^{2}}{cd} \right)}.
   \end{align}   
\end{lemma}

\begin{corollary}\label{cor:coverfinal}
    There exist universal constants $ c'\geq c>0 $ and $ c'\geq 1 $  such that: For a point set $ X=\{ x_{1},\ldots,x_{m} \} \subseteq \cX $ of size $ m $, $ \cH \subseteq \{-1,1\}^{\cX}$ a hypothesis class of VC-dimension $ d $,     $ 0<\alpha<\frac{1}{2} $  and $ 0<\gamma\leq 1 $  such that $ \gamma^{2}\geq \frac{4cd\Ln^{2}(\Ln{ ( \frac{8\alpha m \gamma^{2}}{cd}) }) }{\alpha^{2}m}$, we have 
   \begin{align}
    \ln(\cN_{\infty}(X,\dlh_{\left\lceil\gamma\right\rceil} ,\alpha \gamma)) \leq \frac{c'd}{\alpha^{2}\gamma^{2}}\Ln^{2}\left(\Ln{\left(  \frac{8\alpha m \gamma^{2}}{cd}\right)}\right) \Ln{\left(  \frac{8\alpha m\gamma^{2}}{cd} \right)}=\frac{c'd}{\alpha^{2}\gamma^{2}}\func\left( \frac{8\alpha m \gamma^{2}}{cd}\right)
   \end{align}
   where $ \func(x)=\Ln^{2}(\Ln(x))\Ln(x).$     
\end{corollary}

\begin{proof}[Proof of \cref{cor:coverfinal}]
    Let $ c',c $ be the constants from \cref{lem:coverfinal}, and furthermore consider  $ \gamma^{2}\geq \frac{4cd\Ln^{2}(\Ln{ ( \frac{8\alpha m \gamma^{2}}{cd}) }) }{\alpha^{2}m}.$ To the end of invoking \cref{lem:coverfinal}, notice that for $0< \eps=1/(2\Ln(\Ln{(  \frac{8\alpha m \gamma^{2}}{cd})}) )\leq 1/2,$ we have that $\frac{cd}{\alpha^{2}\eps^{2}m}= \frac{4cd\Ln^{2}(\Ln{ ( \frac{8\alpha m \gamma^{2}}{cd}) }) }{\alpha^{2}m}\leq \gamma^{2},$ and we thus may invoke \cref{lem:coverfinal} with this $ \eps.$   We notice that for this $ \eps $ we have that 
    \begin{align*}
        \Ln^{\eps}{\left(  \frac{8\alpha m\eps^{2} \gamma^{2}}{cd} \right)}= \exp{\left(\eps\ln{\left( \Ln{\left(  \frac{8\alpha m\eps^{2} \gamma^{2}}{cd} \right)}\right)} \right)} \leq  e, 
    \end{align*}
    where we in the inequality used the value of $ \eps $ and that    $ \Ln{(  \frac{8\alpha m\eps^{2} \gamma^{2}}{cd} )}\leq\Ln{(  \frac{8\alpha m \gamma^{2}}{cd} )}$ and $ \ln $ being an increasing function.    
    Thus we conclude that
    \begin{align}
        \ln(\cN_{\infty}(X,\dlh_{\left\lceil\gamma\right\rceil} ,\alpha \gamma)) \leq
        \frac{c'd}{\alpha^{2}\eps^{2}\gamma^{2}} \Ln^{1+\eps}{\left(  \frac{8\alpha m\eps^{2}  \gamma^{2}}{cd} \right)}
        \\ \leq \frac{4ec'd}{\alpha^{2}\gamma^{2}}\Ln^{2}(\Ln{(  \frac{8\alpha m \gamma^{2}}{cd})}) \Ln{\left(  \frac{8\alpha m\gamma^{2}}{cd} \right)}.
    \end{align}
    whereby \cref{cor:coverfinal} follows from redefining $ c' $ to be $ 4ec'.$     
\end{proof}

We now move on to prove \cref{lem:coverfinal}. To prove \cref{lem:coverfinal}, we need the following two lemmas. The first lemma bounds the infinity cover of $ \dlh_{\left\lceil\gamma \right\rceil} $ in terms of the fat shattering dimension of $ \dlh_{\left\lceil\gamma \right\rceil} $.

\begin{lemma}\label{lem:covering}
    There exists universal constants  $ \check{C}\geq 1$ and $\check{c}>0 $ such that: For $ X =\{  x_{1},\ldots,x_{m}\} \subseteq \cX$ a point set of size $ m $, $ \cH \subseteq \{-1,1\}^{\cX}$ a hypothesis class,  $ 0<\eps<1 $, $ 0<\alpha<\frac{1}{2} $  and $ 0<\gamma\leq 1 $ if $\fat_{\check{c} \alpha\eps  \gamma}(\dlh_{\left\lceil\gamma\right\rceil})\leq m $, then we have, for $ d= \fat_{\check{c} \alpha\eps  \gamma}(\dlh_{\left\lceil\gamma\right\rceil}),$   that 
    \begin{align}
        \ln(\cN_{\infty}(X,\dlh_{\left\lceil\gamma\right\rceil} ,\alpha \gamma)) \leq\check{C} 
        d 
        \ln^{1+\eps}\left(\frac{2m}{d  \alpha}\right).
    \end{align}   
\end{lemma}
We also postpone the proof of \cref{lem:covering} and give it after the proof of \cref{lem:coverfinal}. The next lemma gives a bound on the fat shattering dimension of $ \dlh_{\left\lceil\gamma \right\rceil} $, in terms of $ \gamma $ and the VC-dimension of the hypothesis class $ \cH $. 

\begin{lemma}\label{lem:fathelper}
    There exists a universal constant $ C\geq 1 $ such that: For $1\geq \gamma >0  ,$ $ \beta>0,$ and $ \cH \subseteq \{-1,1\}^{\cX}$ a hypothesis class with VC dimension $ d $, we have that  $\fat_{\beta} (\dlh_{\left\lceil\gamma\right\rceil})\leq \frac{Cd}{\beta^{2}} $ 
\end{lemma}
We postpone the proof of \cref{lem:fathelper} for now and give the proof of \cref{lem:coverfinal} by combining \cref{lem:covering} and \cref{lem:fathelper}.

\begin{proof}[Proof of \cref{lem:coverfinal}]
    In what follows let, $ \check{C}\geq1 $ and $ \check{c} >0$ be the universal constants from \cref{lem:covering}, and $ C \geq1$ be the universal constant from \cref{lem:fathelper}. Furthermore, let $ c $ denote the universal constant $ c=\frac{16C}{\check{c}^{2}} $ and $ c'=\max(1,\frac{16\check{C}C}{\check{c}^{2}}) $  in \cref{lem:coverfinal}. Since $ \check{C}\geq 1$ we have $ c'\geq c.$

   By \cref{lem:fathelper}, we have that $ d'=\fat_{\check{c}\alpha\eps\gamma}(\dlh_{\left\lceil\gamma\right\rceil}) $ is upper bounded by $ \frac{Cd}{\check{c}^2\alpha^2\eps^2\gamma^2}$.
   Thus, we have that $  d'=\fat_{\check{c}\alpha\eps\gamma}(\dlh_{\left\lceil\gamma\right\rceil})\leq m$, by the assumption $ \gamma^{2}\geq  \frac{cd}{\alpha^{2}\eps^{2}m}$ and $ c=\frac{16C}{\check{c}^{2}},$ and we may invoke \cref{lem:covering} to obtain 
   
   \begin{align*}
       \ln(\cN_{\infty}(X, \dlh_{\left\lceil\gamma\right\rceil}, \alpha \gamma)) \leq
       \check{C} 
       d'
       \ln^{1+\eps}\left(\frac{2m}{d' \alpha}\right).
   \end{align*}
   Next, we analyze the term $d' \ln^{1+\eps}\left(\frac{2m}{d' \alpha}\right)$. 
   We now notice that the function $ x\ln^{1+\eps}{\left(\frac{a}{x} \right)} $ for    has derivative  $\ln^{\eps}{\left(\frac{a}{x} \right)} \left(\ln{\left(\frac{a}{x} \right)}-1-\eps \right)$, thus the derivative is positive for $ \ln{\left(\frac{a}{x} \right)}\geq 1+\eps $ or equivalently $ \frac{ a }{\exp{\left(1+\eps \right)}}\geq x $, whereby we conclude that $ x\ln^{1+\eps}{\left(\frac{a}{x} \right)} $ is increasing for $ \frac{ a }{\exp{\left(1+\eps \right)}}\geq x $. Now using this with $ a=\frac{2m}{\alpha}\geq \frac{64Cd}{\check{c}^2\alpha^2\eps^2\gamma^2} $, where we used that $ \gamma^{2}\geq  \frac{cd}{\alpha^{2}\eps^{2}m}$, $ c=\frac{16C}{\check{c}^{2}} $   and   $ \alpha < \frac{1}{2} $, and having the upper bound on $ d'=\fat_{\check{c}\alpha\eps\gamma}(\dlh_{\left\lceil\gamma\right\rceil}) $ of $ \frac{Cd}{\check{c}^2\alpha^2\eps^2\gamma^2} \leq \frac{\frac{32Cd}{\check{c}^2\alpha^2\eps^2\gamma^2}}{ \exp{\left(1+\eps \right)} } $, since $  \frac{32}{\exp{\left(1+\eps \right)}}\geq 1  $, we conclude that $ d'=\fat_{\check{c}\alpha\eps\gamma}(\dlh_{\left\lceil\gamma\right\rceil})\leq \frac{Cd}{\check{c}^2\alpha^2\eps^2\gamma^2}\leq \frac{a}{ \exp{\left(1+\eps \right)} }.$ This then implies by the above argued monotonicity of $ x\ln^{1+\eps}{\left(\frac{a}{x} \right)}, $ with $ a=\frac{2m}{\alpha} $  that 
   \begin{align*}
       \check{C} 
       d'
       \ln^{1+\eps}{\left(\frac{2m}{d'\alpha} \right)}
       \leq
       \check{C}\frac{Cd}{\check{c}^2\alpha^2\eps^2\gamma^2} \ln^{1+\eps}{\left(\frac{2m}{\alpha }  \frac{\check{c}^2\alpha^2\eps^2\gamma^2}{Cd} \right)}
       \leq
       \frac{c'd}{\alpha^2\eps^{2}\gamma^{2}} \ln^{1+\eps}{\left(  \frac{8\alpha m\eps^{2}\gamma^{2}}{cd} \right)},
   \end{align*}
   where the last equivalently follows from $ c=\frac{16C}{\check{c}^{2}} $ and $ c'=\max(1,\frac{16\check{C}C}{\check{c}^{2}}) $ (we notice that if $ c' $ is not at least $1$, we could set it equal to $ 1 $ and still have an upper bound)
   which concludes the proof. 
\end{proof}

With the proof of \cref{lem:coverfinal} done, we now provide the proof of \cref{lem:fathelper}. The last part of the proof of \cref{lem:fathelper}, which lower and upper bounds the Rademacher complexity, uses an idea similar to \cite{larsen2022optimalweakstronglearning} [Lemma 9].

\begin{proof}[Proof of \cref{lem:fathelper}]
    First if $ \beta>\gamma,$ then by hypotheses in $ \dlh_{\left\lceil\gamma \right\rceil} $ attaining values in $ [-\gamma,\gamma],$ it follows that $ \fat_{\beta}(\dlh_{\left\lceil\gamma \right\rceil} )=0,$ since no hypothesis in $ \dlh_{\left\lceil\gamma \right\rceil} $ can be above $ r_{1}+\beta $  or below $r_{1}-\beta  $ for any $ r_{1}\in\mathbb{R}.$ As $ \frac{Cd}{\beta}>0 $ this proves the case $ \beta>\gamma,$ thus, we consider for now the case that $ 0<\beta\leq \gamma. $      
    Let $ x_1,\ldots,x_{d'} $ and $ r\in [\beta-\gamma,\gamma-\beta]^{d'} $ be $ \beta $-shattered by $ \dlh_{\left\lceil\gamma\right\rceil} $. We note that we may assume that $ r $ is in  $ [\beta-\gamma,\gamma-\beta]^{d'}, $ as $ [-\gamma,\gamma] $  contains the image of $ \dlh_{\left\lceil\gamma \right\rceil} $. If an entry of $ r_{i} $ is outside the interval $ [\beta-\gamma,\gamma-\beta]^{d'} $, it must either be the case that no function in $ \dlh_{\left\lceil\gamma \right\rceil} $ can be $ \beta $ below or above  $ r_{i} $, depending on $ r_{i} $ being positive or negative. Now by $ x_1,\ldots,x_{d'} $ and $ r\in [\beta-\gamma,\gamma-\beta]^{d'} $ being $ \beta $-shattered by $ \dlh_{\left\lceil\gamma\right\rceil} $, we have that for any $ b\in \left\{ -1,1 \right\}^{d'}  $, that there exists $ \cffs\in\cff  $  such that  
    \begin{align}
      \cffs(x_i)&\leq r_{i}-\beta  \quad \text{ if  } b_{i}=-1
      \\
      \cffs(x_i)&\geq r_{i}+\beta  \quad \text{ if  } b_{i}=1.
    \end{align}
    We now recall that $ \cffs\in \cff $ is generated by a function $ f\in \dlh $ in the following way
    \begin{align*}
            f_{\left\lceil\gamma\right\rceil}(x)=\begin{cases}
                \gamma &\text{ if }f(x)> \gamma
                \\
                -\gamma &\text{ if } f(x)<- \gamma
                \\
                f(x) &\text{ else },
            \end{cases}
    \end{align*} 
    i.e.\ $ f $ always being below $ \cffs $ if $ f $ is less strictly less than $ \gamma $  and $ f $ always being above $ \cffs $ if $ f $ is strictly larger than $ -\gamma $.    
    This implies that the function $ f\in \dlh $ generating $ \cffs $   also satisfies:  
    \begin{align}\label{eq:fathelper1}
        f(x_i)&\leq r_{i} -\beta \quad \text{ if  } b_{i}=-1
        \\
        f(x_i)&\geq r_{i}+\beta  \quad \text{ if  } b_{i}=1,\nonumber
      \end{align}
      as $ r_{i}\in \left[\beta-\gamma,\gamma  -\beta\right] $, we have, if $ b_{i}=-1 $, that  $  \gamma > r_{i}-\beta \geq f_{\left\lceil\gamma\right\rceil}(x_{i}) \geq f(x_{i}) $, and if $ b_{i}=1 $, then $-\gamma< r_{i}+\beta \leq f_{\left\lceil\gamma\right\rceil}(x_{i})\leq f(x_{i}).$ Thus, we conclude that $ \dlh $ also $ \beta$-shatters  $ x_{1},\ldots,x_{d'} $ and $ r_{1},\ldots,x_{d'} $. We now notice that \cref{eq:fathelper1} implies that for all $ b\in \{  -1,1\}^{d'}  ,$ there exists $ f\in \dlh $ such that:

      \begin{align*}
        b_{i}(f(x_{i})-r_{i})\geq \beta.
      \end{align*}
      Using this, and adding $\e_{\sigma\sim \{-1,1  \}^{d'}}\left[ \sum_{i=1}^{d'}  \sigma_{i}(-r_{i})\right]  =0 $ we have the following lower bound on the Rademacher complexity of $ \dlh $ on $ x_{1},\ldots,x_{d'} $:
      \begin{align*}
       \e_{\sigma\sim \{  -1,1\}^{d'} }\left[\frac{1}{d'}\sup_{f\in \dlh} \sum_{i=1}^{d'}  \sigma_{i}f(x_{i})\right]
       =\e_{\sigma\sim \{  -1,1\}^{d'} }\left[\frac{1}{d'}\sup_{f\in \dlh} \sum_{i=1}^{d'}  \sigma_{i} (f(x_{i})-r_{i}) \right]
       \geq \beta.
      \end{align*}

      Furthermore, since $ f\in \dlh $ is a convex combination of hypotheses in $ \cH $, the Rademacher complexity of $ \dlh $ is the same as that of $ \cH $. Due to \cite{dudley} [See, for instance, \cite{rademacherboundlecturenotes}, Theorem 7.2], we have that since $ \cH $ has VC dimension $ d $, the   Rademacher complexity of $ \cH $  is at most $ \sqrt{\frac{Cd}{d'}} $, where $ C\geq 1 $ is a universal constant. Thus, we conclude:
            \begin{align*}
        \e_{\sigma\sim \{  -1,1\}^{d'} }\left[\sup_{f\in \dlh} \sum_{i=1}^{d'}  \sigma_{i}f(x_{i})\right]
        \leq
        \e_{\sigma\sim \{  -1,1\}^{d'} }\left[\sup_{h\in \cH} \sum_{i=1}^{d'}  \sigma_{i}h(x_{i})\right]
        \leq \sqrt{\frac{Cd}{d'}},
      \end{align*}
      which implies that $ d'\leq \frac{Cd}{\beta^{2}} $, as claimed and concludes the proof          
\end{proof}

To prove \cref{lem:covering}, we need the following lemma due to \cite{RudelsonVershynin}. \cite{RudelsonVershynin} bounds the size of a maximal $ \alpha$-packing, which upper bounds the size of a minimal $ \alpha $-cover, since an $ \alpha $-packing is also an $ \alpha $-cover.

\begin{theorem}[\cite{RudelsonVershynin} Theorem 4.4]\label{thm:infinitycoverrundelsonvershynin}
    There exist universal constants $ \check{C}\geq 1$ and $ \check{c} >0$ such that:  For a point set $ X=\{ x_{1},\ldots,x_{m} \}\subseteq \cX  $ of size $ m $ and a function class $ \cF $ defined on $ \cX$, if $ \forall f \in  \cF $, we have that $\sum_{x\in X} |f(x)|/m\leq 1,$ then for $ 0<\eps<1 $, $ 0<\alpha<\frac{1}{2} $, and $ d=\fat_{\check{c}\alpha\eps}( \cF) $ 
    \begin{align*}
     \ln{\left( \cN_{\infty}(X, \cF ,\alpha)   \right)} \leq \check{C}d\ln{\left(\frac{2m}{d\alpha} \right)}\ln^{\eps}{\left(\frac{2m}{d} \right)}.
    \end{align*}        
    
\end{theorem}
We now give the proof of \cref{lem:covering}.

\begin{proof}[Proof of \cref{lem:covering}]
Let $ 0 <\eps<1 $. We consider the function class $ \dlh_{\left\lceil\gamma\right\rceil}/\gamma=\{ f:f=f'/\gamma \text{ for } f'\in \dlh_{\left\lceil\gamma \right\rceil} \} $. 
We note that any function $ f\in \dlh_{\left\lceil\gamma\right\rceil}/\gamma $ is bounded in absolute value by $ 1 $, since any $ f\in \dlh_{\left\lceil\gamma \right\rceil} $ is bounded in absolute value by $ \gamma $.  Thus, we conclude that for $ f\in \dlh_{\left\lceil\gamma\right\rceil}/\gamma $
    \begin{align*}
     \sum_{x\in X}\frac{|f(x)|}{m}\leq 1.
    \end{align*}
    We now invoke \cref{thm:infinitycoverrundelsonvershynin} (using the notation $ d'=\fat_{\check{c}\cdot\eps\cdot \alpha}(\dlh_{\left\lceil\gamma\right\rceil}/\gamma)  $)  to obtain that there exist universal constants $ \check{C} \geq 1$ and $ \check{c}>0$  such that 
    \begin{align*}
        \ln\left( \cN_{\infty}(X,\dlh_{\left\lceil\gamma\right\rceil}/\gamma ,\alpha)   \right)
        &\leq
        \check{C} 
        d'
        \ln\left(\frac{2m}{d'  \alpha}\right)
        \cdot 
        \ln^{\eps}{\left(\frac{2m}{d } \right)}
        \\
        &\leq 
        \check{C} 
        d'
        \ln^{1+\eps}\left(\frac{2m}{d'  \alpha}\right).
        \tag{by $ \alpha<\frac{1}{2} $ }
    \end{align*}
    Now since $  \fat_{\check{c}\cdot\eps\cdot \alpha}\left(\dlh_{\left\lceil\gamma\right\rceil}/\gamma\right)=\fat_{\check{c}\cdot\eps\cdot \alpha\cdot \gamma}(\dlh_{\left\lceil\gamma\right\rceil})$ (so $d'=d= \fat_{\check{c}\cdot\eps\cdot \alpha\cdot \gamma}(\dlh_{\left\lceil\gamma\right\rceil})$ ), we obtain 
    \begin{align*}
        \ln\left( \cN_{\infty}(X,\dlh_{\left\lceil\gamma\right\rceil}/\gamma ,\alpha)    \right)\leq
        \check{C} 
        d
        \ln^{1+\eps}\left(\frac{2m}{d  \alpha}\right).
    \end{align*}
    Furthermore, we note that for any $ f\in \dlh_{\left\lceil\gamma \right\rceil} $ and $ f'\in \Net,$ where $ \Net $ is a minimal (in terms of size) $ \alpha $-cover of $ \dlh_{\left\lceil\gamma \right\rceil}/\gamma  $ on $ X $ in infinity norm,  such that $ f' $ is $ \alpha $-close to $ \frac{f}{\gamma} \in\dlh_{\left\lceil\gamma\right\rceil}/\gamma$, it holds that 
    \begin{align}
     \max_{x\in X} \left| \frac{f(x)}{\gamma}-f'(x) \right| \leq \alpha
     \Rightarrow
     \max_{x\in X} \left|f-\gamma f'(x) \right| \leq \alpha\gamma.
    \end{align}
Thus, we conclude that $ \{ \gamma \cdot f' \}_{f'\in \Net } $, forms an $ \alpha \gamma $  cover of $ \dlh_{\left\lceil\gamma \right\rceil} $. Consequently, we have
    \begin{align}
        \ln\left(\cN_{\infty}(X,\dlh_{\left\lceil\gamma\right\rceil} ,\alpha \gamma)\right)\leq\ln\left(\cN_{\infty}(X,\dlh_{\left\lceil\gamma\right\rceil}/\gamma ,\alpha)\right) \leq
        \check{C} 
        d
        \ln^{1+\eps}\left(\frac{2m}{d  \alpha}\right)
    \end{align}
which concludes the proof. 
\end{proof}

\section{Final Upper bound}\label{sec:finalbound}

In this section, we use \cref{lem:marginbound} and \cref{cor:coverfinal} to give our generalization bound for hypotheses in $ \dlh,$ for all margin levels simultaneously. We now present our generalization bound and then proceed with the proof. 
\begin{theorem}\label{thm:finalmarginbound}
    There exists a universal constant $ C>0 $  such that: For $ \cH \subseteq \{  -1,1\}^{\cX}$ a hypothesis class of VC-dimension $ d $, distribution $ \cD $ over $ \cX\times \{  -1,1\}  $, failure parameter $0<\delta<1$,  and training sequence size $ m\in \mathbb{N} $, it holds with probability at least $ 1-\delta $ over $ \rS\sim \cD^{m} $ that: for any margin $ 0<\gamma\leq 1 $ and any function $ f\in \dlh $
    \begin{align*}
     \ls_{\cD}(f)\leq 
     \ls_{\rS}^{\gamma}(f)
     +\sqrt{C\cdot \ls_{\rS}^{\gamma}(f) \left(\frac{d\func{\left(\frac{m\gamma^{2} }{d} \right)}}{m\gamma^{2}}
     +\frac{\ln{\left(\frac{e}{\delta} \right)}}{m}\right)} 
     + C \left(\frac{d\func{\left(\frac{m\gamma^{2} }{d} \right)}}{m\gamma^{2}}+\frac{\ln{\left(\frac{e}{\delta} \right)}}{m}\right).
    \end{align*}
    where $ \func(x)=\Ln^{2}(\Ln(x))\Ln(x).$      
\end{theorem}

\begin{proof}
    Let $ c' \geq c>0 $ and $ c'\geq 1 $  be the universal constants from \cref{cor:coverfinal} and $ \tilde{c}=\max(1/c,c').$ We will show in the following, that with probability at least $ 1-\delta $ over $ \rS,$ it holds for all $ 0<\gamma\leq 1 $ and all $ f\in \dlh $ that
    \begin{align}\label{eq:finalbound-2}
        \ls_{\cD}(f)
        \leq
        \ls_{\rS}^{\gamma}(f)
        + \cc\Big(\sqrt{\frac{\ls_{\rS}^{\gamma}(f)\cdot 2 \left(\ln{\left(\frac{1640e}{\delta} \right)}+\frac{20000\tilde{c}^2 \min(c,1)d}{32\gamma^{2}} \func{\left(  \frac{32 m\gamma^{2} }{\min(c,1)d} \right)} \right)}{m}}
        \\
        +\frac{5\left(\ln{\left(\frac{1640e}{\delta} \right)}+\frac{20000\tilde{c}^2 \min(c,1)d}{32\gamma^{2}} \func{\left(  \frac{32 m\gamma^{2} }{\min(c,1)d} \right)}\right)}{m}\Big).\nonumber
    \end{align}
    We first notice that if $\frac{20000\tilde{c}^2 \min(c,1)d}{32m}\geq \gamma^{2}$, the right-hand side of the above is greater than
    $ 1 $, and we are done by the left-hand side always being at most
    $ 1 $, i.e. the above inequality holds with probability $ 1 $.
    Thus, we now consider the case where $1\geq \gamma^{2}\geq \frac{20000\tilde{c}^2 \min(c,1)d}{32m} $. We now further restrict the values we have to consider for the bound in \cref{eq:finalbound-2} to be non vacomuous, this is done to the end of employing \cref{cor:coverfinal} later in the argument.
 
    We will soon argue that the function $\func(x)/x$ is strictly decreasing for $ x\geq e^{4},$ furthermore it is also a continuous function for $ x>0.$ 
    We now consider the function $ 20000\tilde{c}^2 \func{\left(  \frac{32 m\gamma^{2} }{\min(c,1)d} \right)}/(\frac{32 m\gamma^{2} }{\min(c,1)d} ),$ which is strictly decreasing in $  \frac{32m\gamma^{2} }{\min(c,1)d}$ when $\frac{32m\gamma^{2} }{\min(c,1)d}\geq e^{4},$ which is the case since we consider $ \gamma $ such that  $ \frac{32 m\gamma^{2} }{\min(c,1)d} \geq 20000\tilde{c}^2\geq 20000.$ 
    This also implies that for the right hand side of \cref{eq:finalbound-2} to be less than $ 1 $ for any $1\geq \gamma^{2}\geq \frac{20000\tilde{c}^2 \min(c,1)d}{32m} $ it must be the case that $ m $ is such that $20000\tilde{c}^2 \func{\left(  \frac{32 m }{\min(c,1)d} \right)}/(\frac{32m }{\min(c,1)d} )< 1$ which we assume is the case from now on.  
    Now let $ \gamma' $ be such that   $ 1\geq\gamma'^{2} \geq \frac{20000\tilde{c}^2 \min(c,1)d}{32m} $ and that $ 20000\tilde{c}^2 \func{\left(  \frac{32 m\gamma'^{2} }{\min(c,1)d} \right)}/(\frac{32 m\gamma'^{2} }{\min(c,1)d} )=1.$ We notice that $ \gamma' $ exists uniquely in this interval since the function $ 20000\tilde{c}^2 \func{\left(  \frac{32 m\gamma^{2} }{\min(c,1)d} \right)}/(\frac{32 m\gamma^{2}}{\min(c,1)d} )$ was strictly decreasing(and continuous) in $  \frac{32 m\gamma^{2} }{\min(c,1)d} $ for $  \frac{32 m\gamma^{2} }{\min(c,1)d} \geq e^{4}$ and for the value $\gamma^{2}=  \frac{20000\tilde{c}^2 \min(c,1)d}{32m}$,  we have that $  \frac{32 m\gamma^{2} }{\min(c,1)d} \geq e^{4}$ and the function evaluated at $\gamma^{2}=  \frac{20000\tilde{c}^2 \min(c,1)d}{32m}$ is greater than $1$ and for $ \gamma=1 $, the function evalutes to something strictly below $ 1. $   Notice that for the values of $ \gamma,$ between $ \gamma'^{2}\geq \gamma^{2} \geq \frac{20000\tilde{c}^2 \min(c,1)d}{32m},$ the bound on the right hand side of \cref{eq:finalbound-2} is again larger than $ 1,$ by the $  20000\tilde{c}^2 \func{\left(  \frac{32 m\gamma^{2} }{\min(c,1)d} \right)}/(\frac{32 m\gamma^{2}}{\min(c,1)d} ) $  being strictly decreasing for these values, so we consider from now on the case $ 1\geq\gamma^{2}\geq\gamma'^{2},$ and $ m $  such that $20000\tilde{c}^2 \func{\left(  \frac{32m }{\min(c,1)d} \right)}/(\frac{32m }{\min(c,1)d} )<1.$ Furthermore, again by the above argued monotonicity we notice that for $ 1\geq\gamma^{2}\geq \gamma'^{2} $ we have that 
    \begin{align}\label{eq:finalbound-3}
        20000\tilde{c}^2 \func{\left(  \frac{32 m\gamma^{2} }{\min(c,1)d} \right)}/\left(\frac{32 m\gamma^{2} }{\min(c,1)d} \right) 
        \leq 20000\tilde{c}^2 \func{\left(  \frac{32 \gamma'^{2}m }{\min(c,1)d} \right)}/\left(\frac{32 \gamma'^{2}m }{\min(c,1)d} \right)= 1
    \end{align}  
    which gonna be important later when we want to invoke \cref{cor:coverfinal}.   

    To argue that $ \func(x)/x $  is strictly decreasing  for $ x\geq e^{4},$ we notice that for $ x\geq e^{4} $ we have that $ \func(x)=\Ln^{2}(\Ln(x))\Ln(x)=\ln^{2}(\ln{\left(x \right)})\ln{\left(x \right)},$ thus it suffices to show that the function $ f(x)=\ln^{2}(\ln{\left(x \right)})\ln{\left(x \right)}/x $ is strictly decreasing for $ x\geq e^{4}.$  
    $ f $ has derivative $ (2\ln{\left(\ln{\left(x \right)} \right)}-\ln{\left(x \right)}(\ln{\left(\ln{\left(x \right)} \right)})^{2}+\left(\ln{\left(\ln{\left(x \right)} \right)}\right)^{2})/x^{2}.$ 
    We observe that for $ x\geq e^{4} $ we have that $  -\ln{\left(x \right)}(\ln{\left(\ln{\left(x \right)} \right)})^{2}/2+\left(\ln{\left(\ln{\left(x \right)} \right)}\right)^{2} < 0$ and for $ x\geq e^{4} $ we have that $ 2\ln{\left(\ln{\left(x \right)} \right)}-\ln{\left(x \right)}(\ln{\left(\ln{\left(x \right)} \right)})^{2}/2\leq 0.$ Whereby we conclude that the derivative of $ f(x)=\ln^{2}(\ln{\left(x \right)})\ln{\left(x \right)}/x $ is negative for $ x\geq e^{4} $ so strictly decreasing for $ x\geq e^{4}.$ We are now ready to setup the parameters for the union bound over the different level sets of $ \gamma $ . 

    Let $I=\{   0,\ldots,\lfloor\log_{2}(  1/\gamma')\rfloor \}$
    and define $ \gamma_{0}^{i}=2^{i}\gamma'  $ and $ \gamma_{1}^{i}=2\gamma_{0}^{i}$, for $ i\in I$, except $ i=\lfloor\log_{2}(  1/\gamma')\rfloor,$ where we define $ \gamma_{1}^{\lfloor\log_{2}(  1/\gamma')\rfloor}=1 $. For $ i\in I $ we also define $ N_{i}= \exp{\left(\frac{72 c'd}{ (\gamma_{0}^{i})^{2}} \func{\left(  \frac{32m(\gamma_{0}^{i})^{2}}{cd} \right)} \right)}  $ and $ \delta_{i}= \frac{\delta}{1640N_{i}^{2}} \frac{d}{m(\gamma_{0}^{i})^{2}}$. Furthermore, for each $ i\in I $ we define $ J_{i} =\{0, 1,\ldots,\big\lfloor\frac{m}{\ln{\left(N_{i} \right)}}\big\rfloor\}$ and for $ j\in J_{i} $  define  $ \tau_{0}^{i,j}=j\frac{\ln{\left(N_{i} \right)}}{m} $ and $ \tau_{1}^{i,j}=\left(j+1\right)\frac{\ln{\left(N_{i} \right)}}{m} $. Lastly for $ i\in I $ and $ j\in J_{i} $ we define
    \begin{align*}
        \beta_{i,j}= \cc\left(\sqrt{\frac{\tau_{1}^{i,j}\cdot 2 \ln{\left(\frac{e}{\delta_{i}} \right)}}{m}} 
         +\frac{2\ln{\left(\frac{e}{\delta_{i}} \right)}}{m}\right).
     \end{align*} 
    For now let  $ i\in I  $, then with the above notation, we get by the union bound and an invocation of \cref{lem:marginbound} that 
    
    \begin{align}
        &\p_{\rS\sim\cD^{m}}\Bigg[\exists j\in J_{i},
            \exists \gamma \in \left[\gamma_{0}^{i},\gamma_{1}^{i}\right]
            ,\exists f\in \dlh: 
            \ls_{\rS}^{\gamma}(f)\in [\tau_{0}^{i,j},\tau_{1}^{i,j}] 
            ,\ls_{\cD}(f) \geq \tau_{1}^{i,j} 
            + \beta_{i,j}
            \Bigg] \nonumber
            \\
            &\leq
            \sum_{j\in J_{i}}\p_{\rS\sim\cD^{m}}\Bigg[
            \exists \gamma \in \left[\gamma_{0}^{i},\gamma_{1}^{i}\right]
            ,\exists f\in \dlh: 
            \ls_{\rS}^{\gamma}(f)\in [\tau_{0}^{i,j},\tau_{1}^{i,j}] 
            ,\ls_{\cD}(f) \geq \tau_{1}^{i,j} 
            + \beta_{i,j}
            \Bigg]\nonumber
            \\
            &\leq \left(\left\lfloor\frac{m}{\ln{\left(N_{i} \right)}}\right\rfloor+1\right) \delta_{i}  \sup_{X\in \cX^{2m}}  \cN_{\infty}(X,\dlh_{\left\lceil2\gamma_{1}^{i}\right\rceil},\frac{\gamma_{0}^{i}}{2}) .\nonumber
    \end{align}
    We remark that if $ i\in I $ is such that  $ \frac{\ln{\left(N_{i} \right)}}{m}\geq1 $, then we have that $ \tau_{1}^{i,j}=(j+1)\frac{\ln{\left(N_{i} \right)}}{m}\geq1,$ which implies, by $ \beta_{i,j}>0 $, and $ \ls_{\cD}(f)\leq 1 $,  that we could have upper bounded the above probability with $0$. Thus, in this case any non-negative number is an upper bound of the above. If $ i\in I $ is such that  $ \frac{\ln{\left(N_{i} \right)}}{m}<1 $, we have that $ \big(\big\lfloor\frac{m}{\ln{\left(N_{i} \right)}}\big\rfloor+1\big) $ is upper bounded by $ \frac{2m}{\ln{\left(N_{i} \right)}} $. Thus, we conclude that
    \begin{align}\label{eq:finalbound-1}
        &\p_{\rS\sim\cD^{m}}\Bigg[\exists j\in J_{i},
            \exists \gamma \in \left[\gamma_{0}^{i},\gamma_{1}^{i}\right]
            ,\exists f\in \dlh: 
            \ls_{\rS}^{\gamma}(f)\in [\tau_{0}^{i,j},\tau_{1}^{i,j}] 
            ,\ls_{\cD}(f) \geq \tau_{1}^{i,j} 
            + \beta_{i,j}
            \Bigg] \nonumber
            \\
            &\leq \frac{2m}{\ln{\left(N_{i} \right)}} \delta_{i}  \sup_{X\in \cX^{2m}}  \cN_{\infty}(X,\dlh_{\left\lceil2\gamma_{1}^{i}\right\rceil},\frac{\gamma_{0}^{i}}{2}) .
    \end{align} 

    To the end of employing \cref{cor:coverfinal} with $\gamma= \min(2\gamma_{1}^{i},1) $ and $ \alpha=\frac{\gamma_{0}^{i}}{4\gamma_{1}^{i}}$, we notice that for $\gamma= \min(2\gamma_{1}^{i},1) $ and $ \alpha=\frac{\gamma_{0}^{i}}{4\gamma_{1}^{i}}$ we have that 
    \begin{align}\label{eq:finalbound-4}
        \cN_{\infty}(X,\dlh_{\left\lceil2\gamma_{1}^{i}\right\rceil},\frac{\gamma_{0}^{i}}{2})\leq\cN_{\infty}(X,\dlh_{\left\lceil\gamma\right\rceil},\alpha \gamma)  
    \end{align}
    where the inequality in the case that $
    \gamma=\min(2\gamma_{1}^{i},1)=1,$ follows from the functions in $ \dlh $ only attaining values in $
    [-1,1]$, so  $
    \dlh_{\left\lceil2\gamma_{1}^{i}\right\rceil}=\dlh_{\left\lceil1\right\rceil}
    =\dlh$ and the covering number being decreasing in the precision argument and $ \gamma_{0}^{i}/2\geq \alpha\gamma $ when $
    \gamma=\min(2\gamma_{1}^{i},1)=1$, in the case that $ \gamma=2\gamma_{1}^{i} $ the above holds with equality. 
    Furthermore, to the end of invoking \cref{cor:coverfinal}(with $ 2m $ points), we now argue that $ \gamma^{2}\geq \frac{4cd\Ln^{2}(\Ln{ ( \frac{8\alpha (2m) \gamma^{2}}{cd}) }) }{\alpha^{2}(2m)}$ for the above $ \gamma$ and $ \alpha .$  To this end we notice that since $ \gamma_{0}^{i} $ and $\gamma_{1}^{i}=2\gamma_{0}^{i},$ except $ i=\lfloor\log_{2}(  1/\gamma')\rfloor,$ where we define $ \gamma_{1}^{\lfloor\log_{2}(  1/\gamma')\rfloor}=1$, we have that $ \frac{1}{8}\leq \alpha=\frac{\gamma_{0}^{i}}{4\gamma_{1}^{i}} \leq \frac{1}{4}.$ Furthermore, we argued previously that it suffices to consider the case that $\frac{20000\tilde{c}^2 \min(c,1)d}{32} \func{\left(  \frac{32 m }{\min(c,1)d} \right)}/{m}<1.$ Now if $ \gamma=1 $ we get that 
    \begin{align*}
      \frac{4cd\Ln^{2}(\Ln{ ( \frac{8\alpha (2m) \gamma^{2}}{cd}) }) }{\alpha^{2}(2m)}\leq \frac{128cd\Ln^{2}(\Ln{ ( \frac{4 m}{cd}) }) }{m} \leq \frac{20000\tilde{c}^2 \min(c,1)d}{32m} \func{\left(  \frac{32 m }{\min(c,1)d} \right)} <1=\gamma^{2}
    \end{align*}
    where we in the first inequality use that $ \frac{1}{8}\leq \alpha \leq \frac{1}{4},$ and in the second inequality use that $ c\leq \tilde{c}=\max(c',1/c),$ where $ c'\geq c $ and $ c'\geq1 $,   and $ 20000/32 \geq128$ and we use in the last inequality that we consider the case $ \gamma=1 $. 
    
    Now if $ \gamma=2\gamma_{1}^{i} $ we get  
    \begin{align*}
        \frac{4cd\Ln^{2}(\Ln{ ( \frac{8\alpha (2m) \gamma^{2}}{cd}) }) }{\alpha^{2}(2m)}\leq \frac{128cd\Ln^{2}(\Ln{ ( \frac{16 m(\gamma_{1}^{i})^{2}}{cd}) }) }{m} 
        &\leq \frac{20000\tilde{c}^2 \min(c,1)d}{32m} \func{\left(  \frac{32 m (\gamma_{1}^{i})^{2}}{\min(c,1)d} \right)}
        \\ &\leq(\gamma_{1}^{i})^2\leq \gamma^{2}
    \end{align*}
    where in the first inequality  we use that $ \frac{1}{8}\leq \alpha \leq \frac{1}{4},$ and  that $ \gamma=2\gamma_{1}^{i} $, and in the second inequality we use that $ c\leq \tilde{c}=\max(c',1/c),$ where $ c'\geq c$ and $ c'\geq 1 $,    and $ 20000/32 \geq128$, furthermore the second to last inequality follows from \cref{eq:finalbound-3} and $ 1\geq\gamma_{1}^{i}\geq\gamma_{0}^{i}\geq\gamma'$ and the last inequality follows by $ \gamma=2\gamma_{1}^{i} $, whereby we have argued that $ \gamma^{2}\geq \frac{4cd\Ln^{2}(\Ln{ ( \frac{8\alpha (2m) \gamma^{2}}{cd}) }) }{\alpha^{2}(2m)}$, and since $ 0<\gamma\leq1 $ and $ 0<\alpha< 1/2 $ we get by invoking  \cref{cor:coverfinal}(with $ 2m $ points) that  
    \begin{align*}
        \cN_{\infty}(X,\dlh_{\left\lceil2\gamma_{1}^{i}\right\rceil},\frac{\gamma_{0}^{i}}{2})
        \leq
        \cN_{\infty}(X,\dlh_{\left\lceil\gamma\right\rceil},\alpha \gamma)   
        &\leq 
        \exp{\left(\frac{c'd}{\alpha^{2}\gamma^{2}}\func\left( \frac{8\alpha (2m) \gamma^{2}}{cd}\right) \right)}  
        \\
        &\leq 
        \exp{\left(\frac{16 c'd}{ (\gamma_{0}^{i})^{2}} \func{\left(  \frac{32 m(\gamma_{0}^{i})^{2}}{cd} \right)} \right)}  
        \leq N_{i} ,
    \end{align*}
     where the first inequality follows from \cref{eq:finalbound-4}, the second inequality by \cref{cor:coverfinal}, the third inequality by  using on the term outside of
     $\func(\cdot)$ that $ \frac{1}{\alpha^{2}\gamma^{2}} =\frac{16(\gamma_{1}^{i})^{2}}{(\gamma_{0}^{i})^{2}\min(4(\gamma_{1}^{i})^{2},1)}\leq
    16/(\gamma_{0}^{i})^{2} $  and using on the term inside of
    $\func(\cdot)$
    that  $ \gamma_{1}^{i}\leq 2\gamma_{0}^{i} $ to upper bound
    $ \alpha
    \gamma^{2}=\frac{\gamma_{0}^{i}}{4\gamma_{1}^{i}}(\min(2\gamma_{1}^{i},1))^{2}\leq
    2(\gamma_{0}^{i})^{2}$.   
    Thus, by using that we concluded that $  \cN_{\infty}(X,\dlh_{\left\lceil2\gamma_{1}^{i}\right\rceil},\frac{\gamma_{0}^{i}}{2})
    \leq N_{i},$ and that $  \delta_{i}= \frac{\delta}{1640N_{i}^{2}} \frac{d}{m(\gamma_{0}^{i})^{2}},$ we get by plugging into the last expression of \cref{eq:finalbound-1}, that 
    \begin{align}\label{eq:finalbound0}
        &\frac{2m}{\ln{\left(N_{i} \right)}}\delta_{i}  \sup_{X\in \cX^{2m}}  \cN_{\infty}(X,\dlh_{\left\lceil2\gamma_{1}^{i}\right\rceil},\frac{\gamma_{0}^{i}}{2}) 
        \leq
        \frac{m}{\ln{\left(N_{i} \right)}}
         \frac{\delta }{820N_{i}}\frac{d}{m(\gamma_{0}^{i})^{2}}.
    \end{align}
     Furthermore, we notice that  by $N_{i} =\exp{\left(\frac{72 c'd}{ (\gamma_{0}^{i})^{2}} \func{\left(  \frac{32m(\gamma_{0}^{i})^{2}}{cd} \right)} \right)},$  we have that
    \begin{align}\label{eq:finalbound1}
        \frac{m}{\ln{\left(N_{i} \right)}} \leq 
     \frac{
        m (\gamma_{0}^{i})^{2} 
        }{
        {72c'd } 
        \func{\left(  \frac{32m(\gamma_{0}^{i})^{2}}{cd} \right)} }
     \leq \frac{m(\gamma_{0}^{i})^{2}}{d},
    \end{align} 
    where the last inequality follows by $ c'\geq 1$ and $ \func{(  \frac{32m(\gamma_{0}^{i})^{2}}{cd} )} \geq1 $. Now, combining \cref{eq:finalbound-1}, \cref{eq:finalbound0} and \cref{eq:finalbound1} we get that 
    \begin{align*}
        &\p_{\rS\sim\cD^{m}}\Big[\exists j\in J_{i},
        \exists \gamma \in \left[\gamma_{0}^{i},\gamma_{1}^{i}\right]
        ,\exists f\in \dlh: 
        \ls_{\rS}^{\gamma}(f)\in [\tau_{0}^{i,j},\tau_{1}^{i,j}] 
        ,\ls_{\cD}(f) \geq \tau_{1}^{i,j} 
        + \beta_{i,j}
        \Big] 
        \\
        &\leq \frac{\delta}{820N_{i}}=\frac{\delta}{820}\exp{\left(-\frac{72 c'd}{ (\gamma_{0}^{i})^{2}} \func{\left(  \frac{32 m(\gamma_{0}^{i})^{2}}{cd} \right)} \right)}.
    \end{align*}
    We showed the above for any $ i \in I=\{   0,\ldots,\lfloor\log_{2}(  1/\gamma')\rfloor \}$; thus, by an application of the union bound, and $\gamma_{0}^{i}= 2^{i}\gamma' ,$   we conclude that: 
    \begin{align*}
        &\p_{\rS\sim\cD^{m}}\Big[\exists i\in I, \exists j\in J_{i},
        \exists \gamma \in \left[\gamma_{0}^{i},\gamma_{1}^{i}\right]
        ,\exists f\in \dlh: 
        \ls_{\rS}^{\gamma}(f)\in [\tau_{0}^{i,j},\tau_{1}^{i,j}] 
        ,\ls_{\cD}(f) \geq \tau_{1}^{i,j} 
        + \beta_{i,j}
        \Big]\\
        &\leq
         \sum_{i=0}^{\lfloor\log_{2}(  1/\gamma')\rfloor}   \frac{\delta}{820}\exp{\left(-\frac{72 c'd}{ (\gamma_{0}^{i})^{2}} \func{\left(  \frac{32 m(\gamma_{0}^{i})^{2}}{cd} \right)} \right)}
         \leq\frac{\delta}{820} 
         \sum_{i=0}^{\lfloor\log_{2}(  1/\gamma')\rfloor}
          \exp{\left(-\frac{72c'd}{2^{2i}\gamma'^{2}} \right)} 
          \\
         &\leq
         \frac{\delta}{820} \sum_{i=0}^{\lfloor\log_{2}(  1/\gamma')\rfloor}
         \exp{\left(-\frac{72c'd}{2^{2(\lfloor\log_{2}(  1/\gamma')\rfloor-i)}\gamma'^{2}} \right)} 
          \leq
         \frac{\delta}{820} \sum_{i=0}^{\lfloor\log_{2}(  1/\gamma')\rfloor}
         \exp{\left(-\frac{72c'd2^{2i}}{2^{2\log_{2}(  1/\gamma' )}\gamma'^{2}} \right)}
         \\
         &\leq
        \frac{\delta}{820} \sum_{i=0}^{\lfloor\log_{2}(  1/\gamma')\rfloor}
        \exp{\left(-72c'd 2^{2i} \right)}
           \leq \frac{\delta}{820},
    \end{align*}
    where the first inequality first inequality follows from the union bound, the second inequality by $ \func{\left(  \frac{32 m(\gamma_{0}^{i})^{2}}{cd} \right)} \geq 1,$  the third inequality by summing in the reverse order, the fourth inequality by $ 2^{\lfloor\log_{2}(  1/\gamma')\rfloor}\leq2^{\log_{2}(  1/\gamma')},$ so  $ -1/2^{\lfloor\log_{2}(  1/\gamma')\rfloor}\leq-1/2^{\log_{2}(  1/\gamma')},$ and the last inequality by the sum being less than 1.
    Thus, we conclude that with probability at least $ 1-\delta$, we have that for all $ i\in I$, for  all $j\in J_{i}$,  for all
    $\gamma \in \left[\gamma_{0}^{i},\gamma_{1}^{i}\right]$,  for all
    $f\in \dlh$, that either 
    \begin{align*}
        \ls_{\rS}^{\gamma}(f)\not\in [\tau_{0}^{i,j},\tau_{1}^{i,j}] 
    \end{align*}
    or
    \begin{align*}
        \ls_{\cD}(f) <\tau_{1}^{i,j} 
        + \beta_{i,j}= \tau_{1}^{i,j} 
        + \cc\left(\sqrt{\frac{\tau_{1}^{i,j}\cdot 2 \ln{\left(\frac{e}{\delta_{i}} \right)}}{m}} 
        +\frac{2\ln{\left(\frac{e}{\delta_{i}} \right)}}{m}\right),
    \end{align*}
    where the last equality uses that $ \beta_{i,j}=\cc\Big(\sqrt{\frac{\tau_{1}^{i,j}\cdot 2 \ln{\left(\frac{e}{\delta_{i}} \right)}}{m}} 
    +\frac{2\ln{\left(\frac{e}{\delta_{i}} \right)}}{m}\Big). $ 
    Furthermore, since $ \cup_{i\in I} [\gamma_{0}^{i},\gamma_{1}^{i}]=[\gamma_{0}^{0},\gamma_{1}^{\lfloor\log_{2}(  1/\gamma')\rfloor}]=[\gamma',1] $ and  for each $ i\in I $,  $ \cup_{j\in J_{i}} [\tau_{0}^{i,j},\tau_{1}^{i,j}]$ contains the interval $ [0,1] $ and $ \ls_{\rS}^{\gamma}(f)\in [0,1] $, the above implies that with probability at least  $ 1-\delta$ over $ \rS, $  for any $ \gamma\in [\gamma',1] $ and for any
    $f\in \dlh$   there exists $ i\in I$, such that
    $\gamma \in \left[\gamma_{0}^{i},\gamma_{1}^{i}\right]$, and $j\in J_{i}$ such that $\ls_{\rS}^{\gamma}(f)\in [\tau_{0}^{i,j},\tau_{1}^{i,j}] $, and
    it holds that
    \begin{align}\label{eq:finalbound8}
        \ls_{\cD}(f) < \tau_{1}^{i,j} 
        + \cc\left(\sqrt{\frac{\tau_{1}^{i,j}\cdot 2 \ln{\left(\frac{e}{\delta_{i}} \right)}}{m}} 
        +\frac{2\ln{\left(\frac{e}{\delta_{i}} \right)}}{m}\right).
    \end{align}

      Now in the case of $\gamma \in \left[\gamma_{0}^{i},\gamma_{1}^{i}\right]$, $\ls_{\rS}^{\gamma}(f)\in [\tau_{0}^{i,j},\tau_{1}^{i,j}] $ and \cref{eq:finalbound8} holding, we get, by the definition of $ \tau_{0}^{i,j} =j\frac{\ln{\left(N_{i} \right)}}{m}$ and $ \tau_{1}^{i,j} =(j+1)\frac{\ln{\left(N_{i} \right)}}{m}$, $ \ln{\left(N_{i} \right)} \leq \ln{\left(\frac{e}{\delta_{i}} \right)}$ and $\sqrt{a+b}\leq \sqrt{a}+\sqrt{b},$ that 
      \begin{align}\label{eq:finalbound5}
        \ls_{\cD}(f)\leq 
        \tau_{1}^{i,j} 
        + \cc\left(\sqrt{\frac{\tau_{1}^{i,j}\cdot 2 \ln{\left(\frac{e}{\delta_{i}} \right)}}{m}} 
        +\frac{2\ln{\left(\frac{e}{\delta_{i}} \right)}}{m}\right)
        \\
        \leq 
        \tau_{0}^{i,j} +\frac{\ln{\left(\frac{e}{\delta_{i}} \right)}}{m}
        + \cc\left(\sqrt{\frac{\tau_{0}^{i,j}\cdot 2 \ln{\left(\frac{e}{\delta_{i}} \right)}}{m}}+ \sqrt{2}\frac{\ln{\left(\frac{e}{\delta_{i}} \right)}}{m}
        +\frac{2\ln{\left(\frac{e}{\delta_{i}} \right)}}{m}\right) \nonumber
        \\
        \leq \ls_{\rS}^{\gamma}(f)
        + \cc\left(\sqrt{\frac{\ls_{\rS}^{\gamma}(f)\cdot 2 \ln{\left(\frac{e}{\delta_{i}} \right)}}{m}}
        +\frac{5\ln{\left(\frac{e}{\delta_{i}} \right)}}{m}\right).\nonumber
      \end{align}
      Furthermore, we have that $ N_{i}= \exp{\left(\frac{72 c'd}{ (\gamma_{0}^{i})^{2}} \func{\left(  \frac{32m(\gamma_{0}^{i})^{2}}{cd} \right)} \right)}  $, $ \delta_{i}= \frac{\delta}{1640N_{i}^{2}} \frac{d}{m(\gamma_{0}^{i})^{2}}$, and by $ \gamma\in \left[\gamma_{0}^{i},\gamma_{1}^{i}\right]$, with $ \gamma_{1}^{i}\leq 2\gamma_{0}^{i} $, it implies that $ \gamma \leq 2\gamma_{0}^{i}$,
      which combined gives that $ \ln{\left(e/\delta_{i} \right)} $  can be bounded by 
      \begin{align}\label{eq:finalbound6}
       \ln{\left(\frac{e}{\delta_{i}} \right)}= \ln{\left(\frac{1640e}{\delta} \right)}+2\ln{\left( N_{i}\right)}+\ln{\left(\frac{m(\gamma_{0}^{i})^{2}}{d} \right)}
       \\
       \leq \ln{\left(\frac{1640e}{\delta} \right)}+\frac{144c'd}{ (\gamma_{0}^{i})^{2}} \func{\left(  \frac{32 m(\gamma_{0}^{i})^{2}}{cd} \right)} +\ln{\left(\frac{m\gamma^{2}}{d} \right)} \tag{by $ \gamma\geq \gamma_{0}^{i} $  }
        \\
        \leq \ln{\left(\frac{1640e}{\delta} \right)}+\frac{576c'd}{ \gamma^{2}} \func{\left(  \frac{32 m\gamma^{2} }{\min(c,1)d} \right)} +\Ln{\left(\frac{m\gamma^{2} }{d} \right)} \tag{by $ \gamma \leq 2\gamma_{0}^{i}$ and $ \gamma_{0}^{i}\leq \gamma $ }.
      \end{align}
      Furthermore by $ c'\geq 1 $ , we conclude that
      \begin{align*}
        \frac{\Ln{\left(\frac{m\gamma^{2}}{d} \right)}}{m}  \leq  \frac{c'd\Ln{\left(\frac{m\gamma^{2}}{d} \right)}}{m\gamma^{2}}\leq \frac{c'd\func{\left(  \frac{32 m\gamma^{2} }{\min(c,1)d} \right)} }{m\gamma^{2}}.
      \end{align*}  
      Now, using this with \cref{eq:finalbound5} and \cref{eq:finalbound6}, we conclude that 
      \begin{align*}
        \ls_{\cD}(f) 
        \leq \ls_{\rS}^{\gamma}(f)
        + \cc\Bigg(\sqrt{\frac{\ls_{\rS}^{\gamma}(f)\cdot 2 \left(\ln{\left(\frac{1640e}{\delta} \right)}+\frac{576c'd}{ \gamma^{2}} \func{\left(  \frac{32 m\gamma^{2} }{\min(c,1)d} \right)} +\Ln{\left(\frac{m\gamma^{2}}{d} \right)} \right)}{m}}
        \\
        +\frac{5\left( \ln{\left(\frac{1640e}{\delta} \right)}+\frac{576c'd}{ \gamma^{2}} \func{\left(  \frac{32 m\gamma^{2} }{\min(c,1)d} \right)} +\Ln{\left(\frac{m\gamma^{2}}{d} \right)}\right)}{m}\Bigg)
        \\
        \leq
        \ls_{\rS}^{\gamma}(f)
        + \cc\Bigg(\sqrt{\frac{\ls_{\rS}^{\gamma}(f)\cdot 2 \left(\ln{\left(\frac{1640e}{\delta} \right)}+\frac{577c'd}{ \gamma^{2}} \func{\left(  \frac{32 m\gamma^{2} }{\min(c,1)d} \right)} \right)}{m}}
        \\
        +\frac{5\left( \ln{\left(\frac{1640e}{\delta} \right)}+\frac{577c'd}{ \gamma^{2}} \func{\left(  \frac{32 m\gamma^{2} }{\min(c,1)d} \right)} \right)}{m}\Bigg).
      \end{align*}
      Thus we conclude that with probability at least $ 1-\delta $ over $ \rS $ it holds for any $ \gamma\in[\gamma',1] $ and any $ f\in \dlh $ that 
      \begin{align}\label{eq:finalbound9}
        \ls_{\cD}(f) 
        \leq
        \ls_{\rS}^{\gamma}(f)
        + \cc\Bigg(\sqrt{\frac{\ls_{\rS}^{\gamma}(f)\cdot 2 \left(\ln{\left(\frac{1640e}{\delta} \right)}+\frac{577c'd}{ \gamma^{2}} \func{\left(  \frac{32 m\gamma^{2} }{\min(c,1)d} \right)} \right)}{m}}
        \\
        +\frac{5\left( \ln{\left(\frac{1640e}{\delta} \right)}+\frac{577c'd}{ \gamma^{2}} \func{\left(  \frac{32 m\gamma^{2} }{\min(c,1)d} \right)} \right)}{m}\Bigg),
      \end{align}
      and since we start by concluding that with probability $ 1 $ over $ \rS $ for any  $ \gamma\in[0,\gamma']$ and any $ f\in\dlh $ \cref{eq:finalbound-2} holds, i.e. 
      \begin{align}\label{eq:finalbound10}
        \ls_{\cD}(f)
        \leq
        \ls_{\rS}^{\gamma}(f)
        + \cc\Big(\sqrt{\frac{\ls_{\rS}^{\gamma}(f)\cdot 2 \left(\ln{\left(\frac{1640e}{\delta} \right)}+\frac{20000\tilde{c}^2 \min(c,1)d}{32\gamma^{2}} \func{\left(  \frac{32 m\gamma^{2} }{\min(c,1)d} \right)} \right)}{m}}
        \\
        +\frac{5\left(\ln{\left(\frac{1640e}{\delta} \right)}+\frac{20000\tilde{c}^2 \min(c,1)d}{32\gamma^{2}} \func{\left(  \frac{32 m\gamma^{2} }{\min(c,1)d} \right)}\right)}{m}\Big)\nonumber
    \end{align}
    and the right hand side of the \cref{eq:finalbound10} being larger than \cref{eq:finalbound9} (we recall that $ \tilde{c}=\max(c',1/c) $ ), this implies that \cref{eq:finalbound10} holds with probability at least $ 1-\delta $ over $ \rS $ for any $ 0<\gamma\leq 1 $ and $f\in \dlh $, which concludes the proof of \cref{thm:finalmarginbound}. 
\end{proof}

\section{Majority of Three Large Margin Classifiers}\label{sec:mainmajoritythree}
In this section we give the proof of \cref{thm:mainmajoritythree}, which implies \cref{cor:maj3intro}. 

To this end we recall some notation. For a target concept $ c^{*} \in \{  -1,1\}^{\cX}  $ we use $ (\cX\times\{  -1,1\})_{c^{*}}^{*} $ for the set of all possible training sequences on $ \cX $, labelled by $ c^{*} $, that is for $ S\in (\cX\times\{  -1,1\})_{c^{*}}^{*} $, any train example $ (x,y)\in S $, is such that $ y=c^{*}(x).$ We remark that $ S $ is seen as a vector/sequence so it may have repetitions of similar train examples. Furthermore, for a distribution $ \cD $ over $ \cX $ we write $ \cD_{c^{*}} $ for the distribution over $ \cX\times\{  -1,1\}  $, defined by $ \p_{(\rx,\ry)\sim\cD_{c^{*}}}\left[(\rx,\ry)\in A\right]=\p_{\rx\sim \cD}\left[(\rx,c^{*}(\rx))\in A\right] $ for any  $ A\subseteq \cX\times\{  -1,1\}$. Furthermore, for $ R\subset \cX $, such that $ \p_{\rx\sim\cD}\left[\rx\in R\right]\not=0 $ , we define $ \cD_{c^{*}}\mid R $ as $$ \p_{(\rx,\ry)\sim\cD_{c^{*}}\mid R}\left[(\rx,\ry)\in A\right]=\p_{\rx\sim \cD}\left[(\rx,c^{*}(\rx))\in A\mid \rx\in R\right]= \frac{\p_{\rx\sim \cD}\left[(\rx,c^{*}(\rx))\in A, \rx\in R\right]}{\p_{\rx\sim \cD}\left[\rx\in R\right]}.$$

We define a learning algorithm $ L $  as a mapping from $ (\cX\times\{  -1,1\})_{c^{*}}^{*} $ to $ \mathbb{R}^{\cX} $, for $ S\in (\cX\times\{  -1,1\})_{c^{*}}^{*} $, we write $ L_{S} $ for short of $ L(S)\in\mathbb{R}^{\cX} $. Furthermore, if $ L_{S} \in \dlh$ for any $ S\in (\cX\times\{  -1,1\})_{c^{*}}^{*} $ we write $ L\in \dlh $ . For $ 0<\gamma<1 $ and target concept $ c^{*} $  we define a $ \gamma $-margin algorithm $ L $ for $ c^{*} $   as a mapping from $ (\cX\times\{  -1,1\})_{c^{*}}^{*} $ to $ \mathbb{R}^{\cX} $, such that for a $ S\in (\cX\times\{  -1,1\})_{c^{*}}^{*} $, we have that $ L_{S}(x)y\geq \gamma $ for all $ (x,y)\in S.$ Furthermore, for three functions $ f_{1},f_{2},f_{3} $ we define $ \maj(f_{1},f_{2},f_{3})=\sign(\sign(f_{1})+\sign(f_{2})+\sign(f_{3})) $, with $ \sign(0)=0.$  

 With the above notation introduced, we can now state \cref{thm:mainmajoritythree}, which is saying that the majority vote of a margin classifier algorithm run on 3 independent training sequences implies the following error bound.
\begin{theorem}\label{thm:mainmajoritythree}
    For distribution $ \cD $ over $ \cX $, target concept $ c^{*} $,  hypothesis class $ \cH \subseteq \{-1,1  \}^{\cX} $ with VC-dimension $ d $,  training sequence size $ m $, margin $ 0<\gamma<1 $ and i.i.d. training sequences $ \rS_{1},\rS_{2},\rS_{3}\sim \cD_{c^{*}}^{m} $, it holds for any $ \gamma $-margin learning algorithm $ L\in \dlh $ for $ c^{*} $ that  
    \begin{align*}
    \e_{\rS_{1},\rS_{2},\rS_{3}\sim \cD^{m}_{c^{*}}}\left[\ls_{\cD_{c^{*}}}(\maj(L_{\rS_{1}},L_{\rS_{2}},L_{\rS_{3}}))\right] =O\left(\frac{d}{\gamma^{2}m}\right).
    \end{align*}
\end{theorem}
Now as AdaBoost is a $ \Omega(\gamma)$-margin learning algorithm when given access to a empirical $ \gamma $-weak learner $ \cW,$ the above \cref{thm:mainmajoritythree} implies \cref{cor:maj3intro}.

To give the proof of \cref{thm:mainmajoritythree} we need the following lemma, which bounds the expected value of two outputs of a $ \gamma $-margin learning algorithm trained on two independent training sequences erring simultaneously. 

\begin{lemma}\label{lem:twofailingsimilar}
    There exists a universal constant $ C\geq1 $ such that: For distribution $ \cD $ over $ \cX $, target concept $ c^{*} $,  hypothesis class $ \cH \subseteq \{-1,1  \}^{\cX} $ with VC-dimension $ d $,  training sequence size $ m $, margin $ 0<\gamma<1 $, i.i.d. training sequences $ \rS_{1},\rS_{2},\rS_{3}\sim \cD_{c^{*}}^{m} $, it holds for any $ \gamma $-margin learning algorithm $ L\in \dlh $ for $ c^{*} $  that  
    \begin{align*}
    \e_{\rS_{1},\rS_{2}\sim \cD^{m}_{c^{*}}}\left[\p_{\rx\sim\cD}\left[   \sign(L_{\rS_{1}}(\rx))\not=c^{*}(\rx),\sign(L_{\rS_{2}}(\rx))\not=c^{*}(\rx)\right]\right] =\frac{96Cd}{\gamma^{2}m}.
    \end{align*}
\end{lemma}

We postpone the proof of \cref{lem:twofailingsimilar} for later in this section and now give the proof of \cref{thm:mainmajoritythree}.

\begin{proof}[Proof of \cref{thm:mainmajoritythree}]
We observe for $ \maj(L_{\rS_{1}},L_{\rS_{2}},L_{\rS_{3}}) $ to fail on an example $ (x,y) $ it must be the case that two of the classifiers err, i.e.\ there exists $ i,j\in\{1,2,3  \}  $, where $ i\not=j $   such that $ \sign(L_{\rS_{i}}(x))\not=y,\sign(L_{\rS_{j}}(x))\not=y $. Thus by a union bound, $ \rS_{1},\rS_{2},\rS_{3} $ being i.i.d.\ and \cref{lem:twofailingsimilar}, we conclude that
\begin{align*}
    &\e_{\rS_{1},\rS_{2},\rS_{3}\sim \cD^{m}_{c^{*}}}\left[\ls_{\cD_{c^{*}}}(\maj(L_{\rS_{1}},L_{\rS_{2}},L_{\rS_{3}}))\right]
    \\
    &\leq \sum_{i>j}\e_{\rS_{i},\rS_{j}\sim \cD^{m}_{c^{*}}}\left[\p_{\rx\sim\cD}\left[   \sign(L_{\rS_{i}}(\rx))\not=c^{*}(\rx),\sign(L_{\rS_{j}}(\rx))\not=c^{*}(\rx)\right]\right]
    =\frac{288Cd}{\gamma^{2}m}
\end{align*} 
which concludes the proof. 
\end{proof}

We now prove \cref{lem:twofailingsimilar}. To the end of showing \cref{lem:twofailingsimilar} we need the following lemma which bounds the conditional error of a large margin learning algorithm under $ \cD_{c^{*}}|R $.

\begin{lemma}\label{lem:expectationlargemarign}
    There exists a universal constant $ c\geq1 $ such that: For distribution $ \cD $ over $ \cX $, target concept $ c^{*} $,  hypothesis class $ \cH \subseteq \{-1,1  \}^{\cX} $ with VC-dimension $ d $,  training sequence size $ m $, margin $ 0<\gamma<1 $, i.i.d. training sequence $ \rS\sim \cD_{c^{*}}^{m} $, subset $ R\subseteq \cX $ such that $ \p_{\rx\sim \cD}\left[R\right]:=\p_{\rx\sim\cD}\left[\rx\in R\right]\not=0 $, it holds for any $ \gamma $-margin learning algorithm $ L\in \dlh $ for $ c^{*} $  that  
    \begin{align*}
        \e_{\rS\sim \cD_{c^{*}}^{m}}\left[\ls_{\cD_{c^{*}}|R}(L_{\rS})\right] \leq\frac{28cd\ln^{2}{\left( \max(e^{2},\frac{\p_{\rx\sim \cD}\left[R\right]\gamma^{2}m}{2d})\right)}}{\p_{\rx\sim \cD}\left[R\right]\gamma^{2}m}.
    \end{align*}
\end{lemma}

We postpone the proof of \cref{lem:expectationlargemarign} to after the proof of \cref{lem:twofailingsimilar}, which we give now. 

\begin{proof}[Proof of \cref{lem:twofailingsimilar}]
    For $ i\in \{  0,1,\ldots \}  $ we define the disjoint regions  $ R_{i}\subseteq \cX $ $ R_{i}=\{x\in \cX:2^{-i-1}< \p_{\rS\sim\cD^{m}_{c^{*}}}\left[ \sign(L_{\rS}(x)) \not = c^{*}(x)  \right]\leq 2^{-i}  \}$  of $ \cX $ (we will write $ \p_{\rx\sim \cD}\left[R_{i}\right] $ for short for $ \p_{\rx\sim\cD}\left[\rx\in R_{i}\right] $). Then by the law of total probability $ \rS_{1} $ and $ \rS_{2} $ being independent we have
    \begin{align}\label{eq:twofailingsimilar2}
        &\e_{\rS_{1},\rS_{2}\sim \cD^{m}_{c^{*}}}\negmedspace\left[\p_{\rx\sim\cD}\left[   \sign(L_{\rS_{1}}(\rx))\not=c^{*}(\rx),\sign(L_{\rS_{2}}(\rx))\not=c^{*}(\rx)\right]\right]
        \negmedspace=\negmedspace\negmedspace\negmedspace\e_{\rx\sim\cD}\negmedspace\left[\p_{\rS\sim \cD^{m}_{c^{*}}}\left[   \sign(L_{\rS}(\rx))\not=c^{*}(\rx)\right]^{2}\right]\nonumber
        \\
        &= \sum_{i=0}^{\infty}  \e_{\rx\sim\cD}\left[\p_{\rS\sim \cD^{m}_{c^{*}}}\left[   \sign(L_{\rS}(\rx)\not=c^{*}(\rx))\right]^{2} \Big| R_{i}\right]\p_{\rx\sim\cD}\left[R_{i}\right]
        \leq
        \sum_{i=0}^{\infty}  2^{-2i}\p_{\rx\sim\cD}\left[R_{i}\right],
    \end{align}
    where the last inequality follows from the definition of $ R_{i}$. We will show that for each $ i\in\{  0,1,2,\ldots\}  $ we have that $ \p_{\rX\sim\cD}\left[R_{i}\right] \leq \frac{8C(i+1)^{2}2^{i}d}{\gamma^{2}m}$ for some universal constant   $ C\geq1 $. Using this, the above gives us that 
    \begin{align*}
        \e_{\rS_{1},\rS_{2}\sim \cD^{m}_{c^{*}}}\left[\p_{\rx\sim\cD}\left[   \sign(L_{\rS_{1}}(\rx))\not=c^{*}(\rx),\sign(L_{\rS_{2}}(\rx))\not=c^{*}(\rx)\right]\right]
        \leq
        \sum_{i=0}^{\infty}  2^{-2i}\frac{8C(i+1)^{2}2^{i}d}{\gamma^{2}m}
        =\frac{96Cd}{\gamma^{2}m},
    \end{align*}
    where the second inequality follows from $ \sum_{i=0}^{\infty}2^{-i}(i+1)^{2} =12$ and this gives the claim of \cref{lem:twofailingsimilar}. 
    We thus proceed to show that for each $ i\in\{  0,1,2,\ldots\}  $, we have that $ \p_{\rX\sim\cD}\left[R_{i}\right] \leq \frac{8C(i+1)^{2}2^{i}d}{\gamma^{2}m}$. 
    To this end let $ i\in\{  0,1,2\ldots\}.$ If $ \p_{\rx\sim\cD}\left[R_{i}\right]=0 $ then we are done, thus we consider the case that $ \p_{\rx\sim\cD}\left[R_{i}\right]\not=0 $.
    We first observe that:
    \begin{align*}
        \e_{\rS\sim \cD_{c^{*}}^{m}}\left[\ls_{\cD_{c^{*}}|R_{i}}(L_{\rS})\right]=\e_{\rx\sim \cD}\left[\p_{\rS\sim \cD_{c^{*}}^{m}}\left[\sign(L_{\rS}(\rx))\not=c^{*}(\rx)\right]\Big|\rx\in R_{i} \right]\geq 2^{-i-1},
    \end{align*}    
    where the equality follows by the definition of $ \sign(0)=0,$ and the inequality follows by the definition of $ R_{i}$.  Furthermore, by \cref{lem:expectationlargemarign} we have that 
    \begin{align*}
        \e_{\rS\sim \cD_{c^{*}}^{m}}\left[\ls_{\cD_{c^{*}}|R_{i}}(L_{\rS})\right]\leq  \frac{28 cd\ln^{2}{\left( \max(e^{2},\frac{\p\left[R\right]\gamma^{2}m}{2d})\right)}}{\p\left[R\right]\gamma^{2}m},
    \end{align*}
    where $ c \geq 1$ is a universal constant.  
    Thus, we conclude that 
    \begin{align}\label{eq:twofailingsimilar1}
        2^{-i-1}\leq \frac{28 cd\ln^{2}{\left( \max(e^{2},\frac{\p\left[R\right]\gamma^{2}m}{2d})\right)}}{\p\left[R\right]\gamma^{2}m}.
    \end{align}
    Now the function $ \frac{\ln^{2}{\left(\max(e^{2},x) \right)}}{x} $ for $ x>e^{2} $ is decreasing since it has derivative $\frac{2\ln{\left(x \right)}-\ln^{2}{\left(x \right)}}{x^{2}}  $, which is negative for $ x>e^{2} $. Furthermore, we have that $ \frac{\ln^{2}(max(e^{2},x))}{x} $ is decreasing for $ 0<x\leq e^{2} $, so   $ \frac{\ln^{2}{\left(\max(e^{2},x) \right)}}{x} $  is decreasing   for $ x>0.$ Now assume for a contradiction that $ \p\left[R\right]\geq \frac{8C(i+1)^{2}2^{i}d}{\gamma^{2}m} $ or equivalently that $ \frac{\p\left[R\right]\gamma^{2}m}{2d}\geq C(i+1)^{2}2^{i}$ for some $ C\geq e^{2} $  to be chosen large enough later.  Thus, since we concluded that $ \frac{\ln^{2}{\left(\max(e^{2},x) \right)}}{x} $  is decreasing   for $ x>0,$ and we assumed for contradiction that $ \frac{\p\left[R\right]\gamma^{2}m}{2d}\geq C(i+1)^{2}2^{i}$ we get that 
    \begin{align*}
        \frac{28 cd\ln^{2}{\left( \max(e^{2},\frac{\p\left[R\right]\gamma^{2}m}{2d})\right)}}{\p\left[R\right]\gamma^{2}m}
        \leq
        \frac{14 c\ln^{2}{\left( \max(e^{2},C(i+1)^{2}2^{i})\right)}}{ C(i+1)^{2}2^{i}}
        \leq \frac{14 c\ln^{2}{\left( C(i+1)^{2}2^{i}\right)}}{ C(i+1)^{2}2^{i}}
    \end{align*}
    where the last inequality follows from $ C\geq e^{2} $, and $ i\geq0.$ Now since it holds that $ \ln{\left(C(i+1)^{2}2^{i} \right)}\leq 3 \max(\ln{\left(C \right)},2\ln{\left(i+1 \right)},i\ln{\left(2 \right)}) $ we get that the following inequality holds   $ \ln^{2}{\left( C(i+1)^{2}2^{i}\right)}\leq 9 \max(\ln^{2}{\left(C \right)},4\ln^{2}{\left(i+1 \right)},i^{2}\ln^{2}{\left(2 \right)}) $. Thus, we conclude that 
    \begin{align*}
        \frac{28 cd\ln^{2}{\left( \max(e^{2},\frac{\p\left[R\right]\gamma^{2}m}{2d})\right)}}{\p\left[R\right]\gamma^{2}m}\leq \frac{126 c \max(\ln^{2}{\left(C \right)},4\ln^{2}{\left(i+1 \right)},i^{2}\ln^{2}{\left(2 \right)})}{ C(i+1)^{2}2^{i}}.
    \end{align*} 
    Furthermore since $ \ln^{2}{\left(i+1 \right)}/((i+1)^{2})\leq 1 $,$ i^{2}/((i+1)^{2})\leq 1 $ and $ \ln{\left(C^{\frac{1}{4}} \right)}\leq\ln{\left(1+C^{\frac{1}{4}} \right)} \leq C^{\frac{1}{4}}$  we conclude that 
    \begin{align*}
        \frac{28 cd\ln^{2}{\left( \max(e^{2},\frac{\p\left[R\right]\gamma^{2}m}{2d})\right)}}{\p\left[R\right]\gamma^{2}m}\leq \frac{126 c \max(\ln^{2}{\left(C \right)},4,\ln^{2}{\left(2 \right)})}{ C2^{i}}\leq \frac{126 c \max(16C^{\frac{1}{2}},4)}{ C2^{i}}\leq \frac{2016c}{C^{1/2}2^{i}}.
    \end{align*}
    where the last inequality follows from $ C\geq e^{2}.$ 
    Since the above is decreasing in $ C $ we get that for $ C=(4\cdot 2016c)^2 $, it holds that
    \begin{align*}
        \frac{28 cd\ln^{2}{\left( \max(e^{2},\frac{\p\left[R\right]\gamma^{2}m}{2d})\right)}}{\p\left[R\right]\gamma^{2}m} <\frac{1}{2^{i+2}},
    \end{align*}  
    which is a contradiction with \cref{eq:twofailingsimilar1}, so it must be the case that $  \p\left[R\right]\leq \frac{8C(i+1)^{2}2^{i}d}{\gamma^{2}m}$, as claimed below \cref{eq:twofailingsimilar2} which concludes the proof.   
\end{proof}

We now give the proof of \cref{lem:expectationlargemarign}.

\begin{proof}[Proof of \cref{lem:expectationlargemarign}]
    If $ \frac{d}{\p\left[R\right]\gamma^{2}m} \geq1$ then we are done by $ \ls_{\cD_{c^{*}}|R} $ always being at most 1, thus for the remainder of the proof we consider the case $ \frac{d}{\p\left[R\right]\gamma^{2}m} <1$.
    
    Define $ \rN=\sum_{(x,y)\in \rS} \ind\{x\in R  \}   $, i.e. the number of examples in $ \rS $, that has its input point in $ R $. We notice that $ \rN $ has expectation $ \p\left[R\right]m $. Thus, by $ \rN $ being a sum of i.i.d. $ \{  0,1\}  $-random variables it follows by an application of Chernoff that 
    \begin{align*}
     \p_{\rS\sim \cD_{c^{*}}^{m}}\left[\rN\leq \p\left[R\right]m/2\right]\leq  \exp{\left(-\p\left[R\right]m/8 \right)}\leq \frac{8}{\p\left[R\right]m},  
    \end{align*} 
    where the last inequality follows from $ \exp(-x)\leq \frac{1}{x} $ for $ x>0 $.
    Thus, from the above and the law of total probability, we conclude that 
    \begin{align}\label{eq:condtionalexpectation-1}
     \e_{\rS\sim \cD_{c^{*}}^{m}}\left[\ls_{\cD_{c^{*}}|R}(L_{\rS})\right] 
      \leq \sum_{i=\left\lceil\p\left[R\right]m/2\right\rceil}^{m} \e_{\rS\sim \cD_{c^{*}}^{m}}\left[\ls_{\cD_{c^{*}}|R}(L_{\rS})|\rN=i \right]\p_{\rS\sim \cD_{c^{*}}^{m}}\left[\rN=i\right] +\frac{8}{\p\left[R\right]m}.
    \end{align}
    For each $ i\in \{\left\lceil\p\left[R\right]m/2\right\rceil,\ldots,m  \} ,$ we will show that $ \e_{\rS\sim \cD_{c^{*}}^{m}}\left[\ls_{\cD_{c^{*}}|R}(L_{\rS})|\rN=i \right]$ is upper bound by $\frac{20 cd\ln^{2}{( \max(e^{2},\frac{\p\left[R\right]\gamma^{2}m}{2d}))}}{\p\left[R\right]\gamma^{2}m} $, where $ c \geq 1$ is a universal constant(the universal constant named $ C $  from \cref{thm:finalmarginbound}), which implies that
    \begin{align*}
        \e_{\rS\sim \cD_{c^{*}}^{m}}\left[\ls_{\cD_{c^{*}}|R}(L_{\rS})\right] \leq  \frac{20cd\ln^{2}{\left( \max(e^{2},\frac{\p\left[R\right]\gamma^{2}m}{2d})\right)}}{\p\left[R\right]\gamma^{2}m}+\frac{8}{\p\left[R\right]m}\leq\frac{28cd\ln^{2}{\left( \max(e^{2},\frac{\p\left[R\right]\gamma^{2}m}{2d})\right)}}{\p\left[R\right]\gamma^{2}m}
    \end{align*}
    as claimed. We now show for each $ i\in \{\left\lceil\p\left[R\right]m/2\right\rceil,\ldots,m  \}  $ that $ \e_{\rS\sim \cD_{c^{*}}^{m}}\left[\ls_{\cD_{c^{*}}|R}(L_{\rS})|\rN=i \right]\leq \frac{20 cd\ln^{2}{( \max(e^{2},\frac{\p\left[R\right]\gamma^{2}m}{2d})))}}{\p\left[R\right]\gamma^{2}m}$.

    Let for now $ i\in \{\left\lceil\p\left[R\right]m/2\right\rceil,\ldots,m  \}  $. First since $ \ls_{\cD_{c^{*}}|R}(L_{\rS}) $ is nonnegative we have that its expectation can be calculated in terms of its cumulative distribution function,
    \begin{align}\label{eq:condtionalexpectation0}
     \e_{\rS\sim\cD_{c^{*}}^{m}}\left[\ls_{\cD_{c^{*}}|R}(L_{\rS})|\rN=i\right]=\int_{0}^{\infty}\p_{\rS\sim \cD_{c^{*}}^{m}}\left[\ls_{\cD_{c^{*}}|R}(L_{\rS})> x|\rN=i \right] \ dx\nonumber
     \\
     \leq \frac{4cd\Ln^{2}{\left( \frac{\gamma^{2}i}{d}\right)}}{\gamma^{2}i}+\int_{\frac{4cd\Ln^{2}{\left( \frac{\gamma^{2}i}{d}\right)}}{\gamma^{2}i}}^{\infty} \p_{\rS\sim\cD^{m}_{c^{*}}}\left[\ls_{\cD_{c^{*}}|R}(L_{\rS})> x|\rN=i \right] \ dx.
    \end{align} 

    We now notice that under the conditional distribution $ \rN=i $, it is the case that the $ \gamma $- margin learning algorithm $ L_{\rS} $, contains $ i $ labelled examples from $ \cD_{c^{*}}|R $. Furthermore, since $L_{\rS}  $ has     $ \ls_{\rS}^{\gamma}(L_{\rS})=0 $ it also has zero margin-loss on the examples drawn from $ \cD_{c^{*}}|R.$ Thus, by invoking \cref{thm:finalmarginbound}, it holds with probability at least $ 1-\delta $ over $ \rS \sim \cD_{c^{*}}^{m}$ conditioned on $ \rN=i $, that  
    \begin{align}\label{eq:condtionalexpectation1}
      \ls_{\cD_{c^{*}}|R}(L_{\rS})\leq \frac{cd\Ln^{2}\left(\frac{\gamma^{2}i}{d}\right)}{\gamma^{2}i}+ \frac{c\ln{\left(\frac{e}{\delta} \right)}}{i}
      \leq \max\left(\frac{2cd\Ln^{2}\left(\frac{\gamma^{2}i}{d}\right)}{\gamma^{2}i}, \frac{2c\ln{\left(\frac{e}{\delta} \right)}}{i}\right)
    \end{align}  
    where we have upper bounded $ \func $ by $ \Ln^{2}$, which holds since for $ x\leq e^{e} $  we have that $\func(x)= \Ln^{2}(\Ln(x))\Ln(x)=\Ln(x)\leq \Ln^{2}(x)$ and for $ x>e^{e} $ $\func(x)= \ln^{2}(\ln(x))\ln(x)\leq \ln^{2}(x)$, and used that $ a+b\leq 2\max(a,b) $ for $ a,b\geq0 $, and $ c\geq 1$ being the universal constant of \cref{thm:finalmarginbound}. 
    For $ x>2c/i $, we now notice that if we set $ \delta=e\cdot \exp(-ix/2c),$ which is strictly less than $1  $ by $ x>2c/i,$ we conclude from \cref{eq:condtionalexpectation1}  that
    \begin{align*} 
        \p_{\rS\sim \cD^{m}_{c^{*}}}\left[\ls_{\cD_{c^{*}}|R}(L_{\rS})
        > \max\left(\frac{2cd\Ln^{2}\left(\frac{\gamma^{2}i}{d}\right)}{\gamma^{2}i},x\right)\Big|\rN=i\right]\leq e\cdot \exp(-ix/2c),
      \end{align*}  
    Furthermore, we notice that if $ 0<x\leq 2c/i $ then the above right-hand side is at least $ 1,$ which is also upper bounding the left-hand side since it is at most $ 1,$ thus the above holds for any  $ x>0. $ 
    Thus, plugging this into \cref{eq:condtionalexpectation0} we get that 
    \begin{align*}
        \e_{\rS\sim\cD_{c^{*}}^{m}}\left[\ls_{\cD_{c^{*}}|R}(L_{\rS})|\rN=i\right]
        &\leq \frac{4cd\Ln^{2}{\left( \frac{\gamma^{2}i}{d}\right)}}{\gamma^{2}i}+\int_{\frac{4cd\Ln^{2}{\left( \frac{\gamma^{2}i}{d}\right)}}{\gamma^{2}i}}^{\infty} \p_{\rS\sim\cD^{m}_{c^{*}}}\left[\ls_{\cD_{c^{*}}|R}(L_{\rS})> x|\rN=i \right] \ dx
        \\
        &\leq \frac{4cd\Ln^{2}{\left( \frac{\gamma^{2}i}{d}\right)}}{\gamma^{2}i}+\int_{\frac{4cd\Ln^{2}{\left( \frac{\gamma^{2}i}{d}\right)}}{\gamma^{2}i}}^{\infty} e\cdot \exp(-ix/2c) \ dx 
        \\
        &\leq \frac{4cd\Ln^{2}{\left( \frac{\gamma^{2}i}{d}\right)}}{\gamma^{2}i}+
        \frac{2ec}{i} \cdot \exp\left(-\frac{4cd\Ln^{2}{\left( \frac{\gamma^{2}i}{d}\right)}}{2c\gamma^{2}} \right) 
        \\
        &\leq 
        \frac{10cd\Ln^{2}{\left( \frac{\gamma^{2}i}{d}\right)}}{\gamma^{2}i}
        \leq\frac{10cd\ln^{2}{\left( \max(e^{2},\frac{\gamma^{2}i}{d})\right)}}{\gamma^{2}i},
    \end{align*}
    where we have used in the third inequality that $ \int\exp(-ax) \ dx= -\exp(-ax)/a +C$, in the second to last inequality that $ 4+2e\leq 10 $, and in the last that $ \Ln(x)=\ln{\left(max(e,x) \right)}\leq\ln{\left(\max(e^{2},x)  \right)}.$ 
    Now for $ \frac{\gamma^{2}i}{d}\leq e^{2}  $, $ \frac{10cd\ln^{2}{( \max(e^{2},\frac{\gamma^{2}i}{d}))}}{\gamma^{2}i} $  is a decreasing function in $ \frac{\gamma^{2}i}{d}$. 
    Furthermore, since $ \ln^2{\left(x \right)}/x $ has derivative $ \frac{2\ln{\left(x \right)}-\ln^2{\left(x \right)}}{x^2} $ we conclude that $ \ln^2{\left(x \right)}/x $ is decreasing for $ x\geq e^{2} $, whereby we conclude that $ \frac{10cd\ln^{2}{( \max(e^{2},\frac{\gamma^{2}i}{d}))}}{\gamma^{2}i} $ is decreasing in $ \frac{\gamma^{2}i}{d} $ for $ \frac{\gamma^{2}i}{d} \geq e^{2}$, so we conclude that for $ \frac{\gamma^{2}i}{d} >0 $ the function $ \frac{10cd\ln^{2}{( \max(e^{2},\frac{\gamma^{2}i}{d}))}}{\gamma^{2}i} $ is decreasing in $ \frac{\gamma^{2}i}{d}  $. 
    Now $ i\geq \left\lceil\p\left[R\right]m/2\right\rceil\geq \p\left[R\right]m/2 >0$ thus we have $ \frac{\gamma^{2}i}{d}\geq \frac{\p\left[R\right]\gamma^{2}m}{2d}$, which by the above argued monotonicity implies that $ \frac{10cd\ln^{2}{( \max(e^{2},\frac{\gamma^{2}i}{d}))}}{\gamma^{2}i} \leq \frac{2\cdot10 cd\ln^{2}{( \max(e^{2},\frac{\p\left[R\right]\gamma^{2}m}{2d}))}}{\p\left[R\right]\gamma^{2}m} $. 
    Thus, we have argued that 
    \begin{align*}
        &\e_{\rS\sim\cD_{c^{*}}^{m}}\left[\ls_{\cD_{c^{*}}|R}(L_{\rS})|\rN=i\right]
        \leq \frac{20cd\ln^{2}{( \max(e^{2},\frac{\p\left[R\right]\gamma^{2}m}{2d}))}}{\p\left[R\right]\gamma^{2}m},
    \end{align*}
    which shows the claim below \cref{eq:condtionalexpectation-1} and concludes the proof. 
\end{proof}

\section{Acknowledgments}
Mikael M\o ller H\o gsgaard is funded by DFF Sapere Aude
  Research Leader Grant No. 9064-00068B by the Independent Research
  Fund Denmark. Kasper Green Larsen is co-funded by DFF Grant No. 9064-00068B and co-funded by the European Union (ERC, TUCLA, 101125203). Views and opinions expressed are however those of the author(s) only and do not necessarily reflect those of the European Union or the European Research Council. Neither the European Union nor the granting authority can be held responsible for them.

\bibliography{refs.bib}
\bibliographystyle{apalike}

\end{document}